  \documentclass{article}

  \IfFileExists{microtype.sty}{\usepackage{microtype}}{}
  \usepackage{graphicx}
  \IfFileExists{subfigure.sty}{\usepackage{subfigure}}{}
  \usepackage{placeins}
  \usepackage{booktabs} % for professional tables
  \usepackage{caption}

  \usepackage{amsthm}
  \usepackage{thmtools, thm-restate}
  \usepackage{bm}
  \usepackage{xcolor}
  \usepackage{enumitem}
  \usepackage{hyperref}

  \usepackage{algorithm}
  \usepackage{algorithmic}

  \usepackage[final,main]{neurips_2026}

  \usepackage{amsmath}
  \usepackage{amssymb}
  \usepackage{mathtools}

  \IfFileExists{cleveref.sty}{%
    \usepackage[capitalize,noabbrev]{cleveref}%
  }{%
    \newcommand{\cref}[1]{\ref{##1}}%
    \newcommand{\Cref}[1]{\ref{##1}}%
  }

  \newcommand{\papertitle}{BONSAI: Bayesian Optimization with Natural Simplicity and Interpretability}

  \theoremstyle{plain}
  \newtheorem{theorem}{Theorem}[section]
  \newtheorem{proposition}[theorem]{Proposition}
  \newtheorem{lemma}[theorem]{Lemma}
  
  \theoremstyle{definition}
  
  \newtheorem{assumption}[theorem]{Assumption}
  \newtheorem{remark}[theorem]{Remark}

  \DeclareMathOperator*{\argmax}{arg\,max}

  \title{\papertitle}

  \author{
    Samuel Daulton\thanks{Corresponding author: \texttt{sdaulton@meta.com}} \\
    Meta \\
    \And
    David Eriksson \\
    Meta \\
    \And
    Maximilian Balandat\thanks{Work done at Meta.} \\
    Atomic Machines \\
    \And
    Eytan Bakshy \\
    Meta \\
  }

\begin{document}

  \maketitle

  \begin{abstract}
  Bayesian optimization (BO) is a popular technique for sample-efficient optimization of black-box functions. In many applications, the parameters being tuned come with a carefully engineered default configuration, and practitioners only want to deviate from this default when necessary. Standard BO, however, does not aim to minimize deviation from the default and, in practice, often pushes weakly relevant parameters to the boundary of the search space.
  This makes it difficult to distinguish between important and spurious changes and increases the burden of vetting recommendations when the optimization objective omits relevant operational considerations.
  We introduce BONSAI, a default-aware BO policy that prunes low-impact deviations from a default configuration while explicitly controlling the loss in acquisition value. BONSAI is compatible with a variety of acquisition functions, including expected improvement and upper confidence bound (GP-UCB). We theoretically bound the regret incurred by BONSAI, showing that, under appropriate conditions, it retains the no-regret property of vanilla GP-UCB and removes irrelevant changes.
  % Assuming known ARD lengthscales, BONSAI provably removes irrelevant changes at zero acquisition cost, while retaining changes that materially contribute to the acquisition objective.
  %Thus, in this setting BONSAI retains the standard GP-UCB regret rate while returning candidates supported only on truly relevant coordinates.
  Across many real-world applications, we empirically find that BONSAI substantially reduces the number of non-default parameters in recommended configurations while maintaining competitive optimization performance with little effect on wall time. Its candidate-generation cost averages only $1.5\times$ that of standard BO, compared with $7$--$34\times$ for prior sparse-BO methods.
  %\end{abstract}

  %Bayesian optimization (BO) is a popular technique for sample-efficient optimization of black-box functions. In many applications, the parameters being tuned come with a carefully engineered default configuration, and practitioners only want to deviate from this default when necessary. Standard BO, however, does not aim to minimize deviation from the default and in practice will often push weakly relevant parameters to the boundary of the search space. This makes it difficult to distinguish between important and spurious changes and can lead to unintended consequences when the optimization objective does not fully capture all relevant considerations, e.g., system stability. We introduce BONSAI, a default-aware BO policy that prunes low-impact deviations from a default configuration while explicitly controlling the loss in acquisition value. BONSAI is compatible with a variety of acquisition functions, including expected improvement and upper confidence bound (GP-UCB). We theoretically bound the regret incurred by BONSAI, showing that under certain conditions, it enjoys the same no-regret property of vanilla GP-UCB. Across many real-world applications, we empirically find that BONSAI substantially reduces the number of non-default parameters in recommended configurations while matching or outperforming vanilla BO with respect to optimization performance, with little effect on wall time.
  \end{abstract}

  %%%%%%%%%%%%%%%%%%%%%%%%%%%%%%%%%%%%%%%%%%%%%%%%%%%%%%%%%%%%%%%%%%%%%%
  \section{Introduction}
  \label{sec:intro}
  Bayesian optimization (BO) \citep{garnett_bayesoptbook_2023} has become a popular approach for optimizing expensive and noisy black-box functions, with applications ranging from hyperparameter tuning to physical experiment design and system configuration. BO relies on a probabilistic surrogate model of an unknown objective~$f$, which is used in an acquisition function to select informative points to evaluate next.

  In many of these applications, there is a distinguished \emph{default} or \emph{status quo} configuration. For instance, production systems often ship with carefully tuned configurations, compiler flags, or model-serving infrastructure. Deviating from such defaults can be costly: changing many parameters at once can introduce unintended behaviors that are not captured in the optimization objective and can increase technical debt \citep{scully2015debt, sebo}.
  In production A/B tests or serving systems, each changed parameter expands the surface that must be vetted for regressions and unmodeled guardrails. Similar considerations arise when adjusting cell-network transmission power and downtilt or validating changes to optical-component geometry: a recommendation with fewer changes is easier to review, test, and deploy.

  Practitioners, therefore, frequently ask a question that standard BO does not address directly:
  \emph{``Given a current default configuration, what is the minimal set of changes I can make to optimize my system?''}
  Unfortunately, conventional BO is indifferent to this preference; it typically adjusts many weakly relevant parameters and often pushes marginal dimensions to the boundary of the search space, even when their effect on the objective is negligible.

  We propose BONSAI (Bayesian Optimization with Natural Simplicity and Interpretability) to address this gap. BONSAI is a lightweight candidate-transformation layer that sits on top of any acquisition function. At each BO iteration $t$, BONSAI: 1) identifies a base candidate $\bm x_t^*$ that maximizes the acquisition function $\alpha_t$, 2) defines the \emph{acquisition gap} $\Delta_t(\bm x) = \alpha_t(\bm x_t^*) - \alpha_t(\bm x)$ of any point $\bm x$ at round $t$, 3) greedily resets components of $\bm x_t^*$ back to their default values as long as the acquisition gap of the pruned candidate remains below a relative threshold, and 4) returns the pruned point $\tilde{\bm x}_t$ when any further one-component reset exceeds the threshold. Because BONSAI directly affects the queried points, it is part of the BO decision policy rather than a purely post-hoc analysis. BONSAI produces candidates with fewer deviations from the default, making the proposed changes easier to interpret than the dense recommendations often produced by vanilla BO.

  \paragraph{Contributions.}
  % This paper makes the following contributions.
  \begin{enumerate}[leftmargin=14pt, labelwidth=!, labelindent=0pt]
      \item We formalize the setting of \emph{default-aware} BO, where the goal is not only to optimize a black-box function but also to minimize deviation from a specified default configuration, measured by the number of components that change.

      \item We introduce BONSAI, a default-aware BO policy that modifies acquisition-maximizing candidates by reverting low-impact deviations to a user-specified default configuration while enforcing a gap rule on the decrease in acquisition value.

      \item In the context of using the Upper Confidence Bound (UCB) acquisition function with a Gaussian Process surrogate, we prove that BONSAI's regret is bounded by the usual GP-UCB term plus a sum of threshold-dependent penalties, the latter of which is controlled under a particular gap rule leading to sublinear regret.

      \item We prove \emph{sparsity recovery} under an idealized known lengthscale assumption and a relevant-coordinate gap condition: BONSAI removes every irrelevant coordinate modified by the base acquisition candidate and preserves every modified relevant coordinate. These resets incur zero acquisition loss, so BONSAI retains the standard GP-UCB regret rate in this setting. Existing sparse-BO methods (IR, ER, SEBO) provide neither a regret bound nor an analogous support-recovery guarantee.

      % \item We discuss how the thresholds control the trade-off between staying close to the acquisition maximizer and allowing more aggressive pruning.

      \item We evaluate Expected Improvement (EI) and UCB variants of BONSAI on synthetic and real-world test problems, demonstrating that BONSAI yields favorable sparsity–performance trade-offs: it substantially reduces the number of non-default parameters in the recommended configurations while maintaining competitive optimization performance and wall time relative to standard BO.
  \end{enumerate}

  %%%%%%%%%%%%%%%%%%%%%%%%%%%%%%%%%%%%%%%%%%%%%%%%%%%%%%%%%%%%%%%%%%%%%%
  \section{Background}
  \label{sec:background}

  We briefly review Bayesian optimization (BO), Gaussian process (GP) surrogates, and the GP-UCB algorithm, focusing on the components used in our analysis. We refer to Appendix~\ref{appdx:theory} and \citet{srinivas} for full details.

  \paragraph{Bayesian optimization and GP surrogates.}
  We consider the problem of maximizing an unknown function $f : \mathbb{X} \rightarrow \mathbb{R}$ over a compact domain $\mathbb{X} \subset \mathbb{R}^d$. At iteration $t$, the algorithm selects a point $\bm x_t \in \mathbb{X}$, observes a noisy evaluation
  $y_t = f(\bm x_t) + \varepsilon_t,$
  where the noise $\varepsilon_t \sim \mathcal N(0, \sigma^2)$ and $\sigma^2$ is the noise variance, and uses the history $\{(\bm x_s, y_s)\}_{s=1}^t$ to decide where to evaluate next.
  A common choice of surrogate model is a GP prior over $f$ with a constant mean and covariance function $k$. Conditioned on the data up to time $t-1$, the GP posterior has mean $\mu_{t-1} : \mathbb X \to \mathbb R$ and standard deviation $\sigma_{t-1} : \mathbb X \to \mathbb R_+$.

  \paragraph{Acquisition functions.}
  BO algorithms typically select $\bm x_t$ by maximizing an \emph{acquisition function} $\alpha_t : \mathbb X \to \mathbb R$ that quantifies the utility of evaluating $f$ at $\bm x$ given the current posterior. The acquisition balances \emph{exploration} by preferring points with large posterior uncertainty and \emph{exploitation} by preferring points with a large posterior mean.
  Popular choices include Expected Improvement (EI, \citealp{jones98}), which measures the expected positive improvement over the best observed value so far; Upper Confidence Bound (UCB, \citealp{srinivas}), which takes the form
  $
  \alpha_t^{\mathrm{UCB}}(\bm x) = \mu_{t-1}(\bm x) + \beta_t\, \sigma_{t-1}(\bm x),
  $
  where $\beta_t > 0$ controls the exploration–exploitation trade-off; and information-based criteria such as entropy search~\citep{entropy_search}.
  In this paper, BONSAI itself is agnostic to the choice of acquisition function; it only requires the ability to evaluate $\alpha_t(\bm x)$ at candidate points. For our theoretical analysis, we focus on UCB, while in our experiments we also apply BONSAI to EI, which is often more robust in practice \citep{jones98, bo_review}.

  %%%%%%%%%%%%%%%%%%%%%%%%%%%%%%%%%%%%%%%%%%%%%%%%%%%%%%%%%%%%%%%%%%%%%%
  \section{Related Work}
  \label{sec:related_work}

  \paragraph{Other related work.}
  \emph{Conservative/safe BO}~\citep[e.g.,][]{kazerouni2017conservative,sui2015safe} constrains the unknown objective relative to a baseline; BONSAI instead constrains the acquisition function and simplifies recommendations in input space, but the two are complementary. \emph{Prior-based BO} methods such as PiBO~\citep{hvarfner2022pibo} augment the acquisition function with a user-specified prior often centered at a known-good default, providing a soft probabilistic bias toward the default. In contrast, BONSAI imposes a hard $\ell_0$ pruning step with explicit acquisition-gap control: rather than reweighting the acquisition surface, it certifies that any deviation from the default contributes meaningfully to the acquisition value. The two are complementary—a PiBO-style prior could be combined with BONSAI pruning. \emph{Variable selection and shrinkage} methods (SAASBO~\citep{saasbo}, VSBO~\citep{vsbo}, dimensionality-scaled priors~\citep{hvarfner2024dsp}) target the surrogate's structure to improve sample efficiency, but their recommendations can still differ from the default in many components. \emph{Explainable BO}~\citep{chakraborty2025explainablebayesianoptimization,pmlr-v238-adachi24a} explains an existing recommendation; BONSAI instead simplifies the recommendation directly. See Appendix~\ref{appdx:additional_related_work} for further discussion.

  \newlength{\acqimgheight}
  \settoheight{\acqimgheight}{\includegraphics[width=0.55\linewidth]{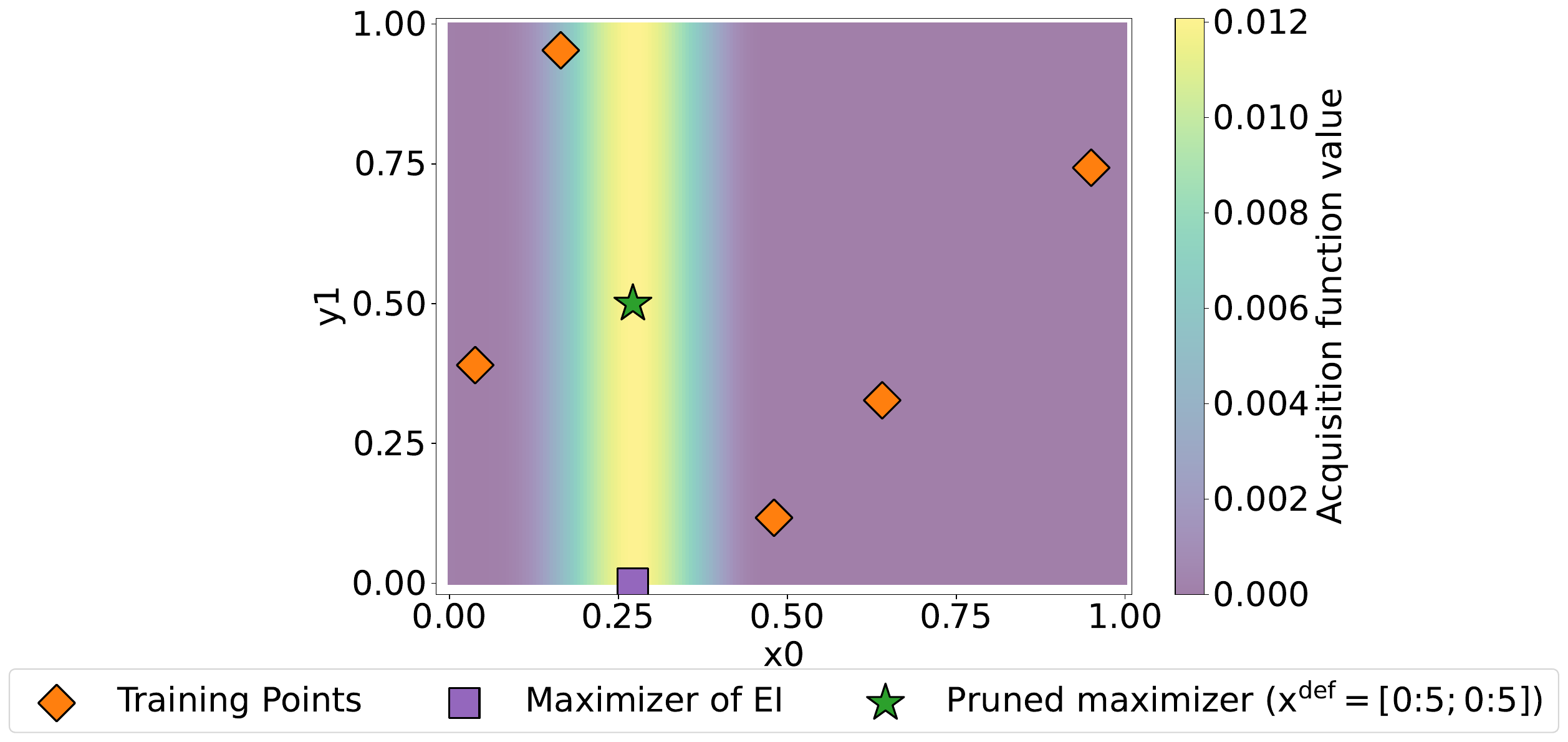}}
  \begin{figure}[!ht]
    \centering
    \begin{minipage}[t]{0.49\linewidth}\centering
      \includegraphics[width=\linewidth]{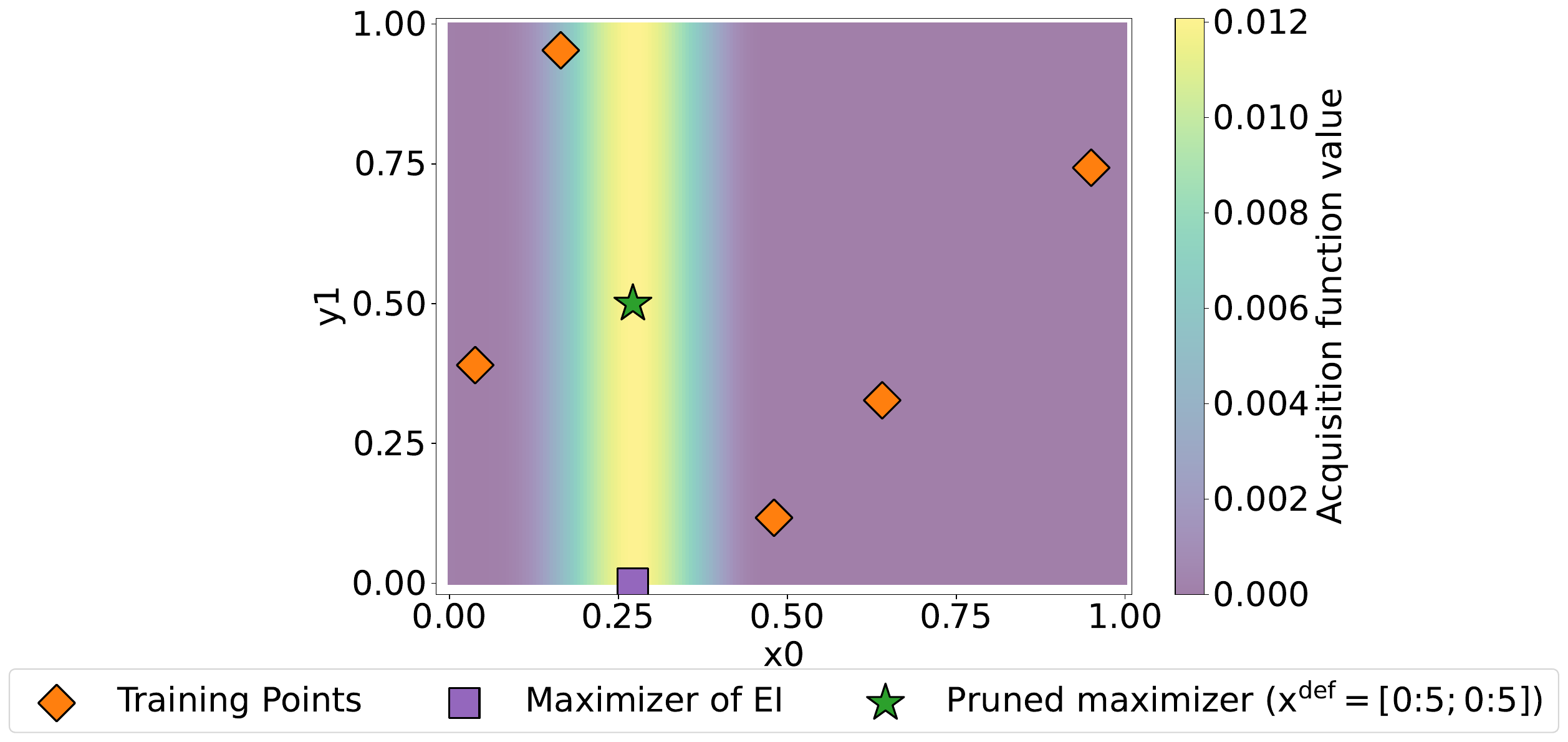}
      \captionof{figure}{Example where a parameter with no effect ($x_1$) is sent to the boundary by BO. %With 5 observations from $f(\bm x) = (x_0 - 0.25)^2$, a GP correctly infers that only $x_0$ is relevant.
      Even though the effect of $x_1$ is close to zero, %(the GP lengthscale for $x_1$ is long)
      EI is maximized with $x_1$ on the boundary (purple square). BONSAI mitigates this by moving $x_1$ back to its default (green star).}
      \label{fig:motivating figure}
    \end{minipage}\hfill
    \begin{minipage}[t]{0.49\linewidth}\centering
      \vbox to \acqimgheight{\vfil
        \centering
\begin{small}
\begin{sc}
\resizebox{\columnwidth}{!}{%
\begin{tabular}{lcc}
\toprule
Method & Single-Obj.\ (5) & Multi-Obj.\ (3) \\ \midrule
IR-$L_0$ & 7.4x ($\pm$ 8.8x) & N/A \\
ER-$L_0$ & 8.8x ($\pm$ 9.3x) & 11.1x ($\pm$ 14.8x) \\
SEBO & 14.1x ($\pm$ 19.4x) & 35.2x ($\pm$ 63.9x) \\
BONSAI EI & $\bm{1.6\text{x}}$ ($\pm$ 0.8x) & $\bm{1.5\text{x}}$ ($\pm$ 0.5x) \\
PiBO & 2.1x ($\pm$ 0.6x) & 2.4x ($\pm$ 1.1x) \\
\bottomrule
\end{tabular}%
}
\end{sc}
\end{small}

      \vfil}
      \captionof{table}{Average generation time relative to Vanilla BO across the problems in Sec.~\ref{sec:experiments} ($\pm$ 2 standard errors over problems), split by single-objective and multi-objective. IR is not applicable to multi-objective settings. Fastest method per column in bold.}
      \label{tab:avg_gen_times}
    \end{minipage}
  \end{figure}

  \paragraph{Sparse BO.}
  SEBO~\citep{sebo} poses sparse BO as learning a sparsity--objective trade-off; we instead optimize the objective and remove low-impact changes from the default. The internal and external regularization (IR/ER) methods of \citet{sebo} are most related: IR penalizes the objective by $\|\bm x-\bm x^\text{def}\|_0$ (requiring acquisition-specific modifications) and ER penalizes the acquisition. Both rely on a hard-to-tune regularization parameter, are non-differentiable due to the $\ell_0$ term, and lack theoretical guarantees. They also rely on homotopy continuation, making them substantially more expensive than standard BO: across the benchmarks in Sec.~\ref{sec:experiments}, IR/ER/SEBO incur average slowdowns of roughly $7\times$/$9\times$/$14\times$ over Vanilla BO on single-objective problems (and ER/SEBO incur $11\times$/$35\times$ on multi-objective; IR is not applicable to multi-objective), with per-problem slowdowns exceeding $25\times$ (Table~\ref{tab:avg_gen_times}). On the main-figure problems, BONSAI averages $1.6\times$ for single-objective and $1.5\times$ for multi-objective tasks.

  \citet{pmlr-v286-kim25b} analyze regret under inexact acquisition maximization but do not consider it as a mechanism for simplifying recommendations.

  See Appendix~\ref{appdx:additional_related_work} for additional related work.
  %%%%%%%%%%%%%%%%%%%%%%%%%%%%%%%%%%%%%%%%%%%%%%%%%%%%%%%%%%%%%%%%%%%%%%
  \section{Problem Setting}
  \label{sec:problem_setting}
  We now formalize the default-aware BO setting. For clarity, we first consider a single BO iteration and suppress the time index on the acquisition function, writing $\alpha(\cdot)$ to denote the acquisition function in the current iteration. Each configuration $\bm x \in \mathbb X \subset \mathbb R^d$ represents a setting of $d$ controllable input parameters, and we refer to $x_j$ as the $j$-th component of the configuration.

  We assume the standard BO setting from \cref{sec:background}. In addition we are given a fixed \emph{default configuration} $\bm x^{\mathrm{def}} \in \mathbb X$. Intuitively, $\bm x^{\mathrm{def}}$ represents the status quo: a configuration that has been vetted or previously deployed. Deviating from this default can introduce unintended system behaviors or technical debt that is not captured by the primary objective $f$, and hence, we prefer to deviate from $\bm x^{\mathrm{def}}$ as little as possible.

  For any candidate $\bm x \in \mathbb{X}$, we define its active set $A(x) := \{j \in \{1,\dots,d\} : x_j \neq \bm x^\text{def}_j\}$,
  which represents the components in which $x$ differs from the default. The complexity or deviation of a configuration can thus be measured by its $\ell_0$ distance to the default, $\|\bm x -\bm x^\text{def}\|_0 = |A(\bm x)|$.
  Our goal is to find a configuration that is both high-performing and sparse relative to the default. This quantity ignores the \emph{magnitude} of the changes and only counts how many coordinates differ from the default, matching our goal of highlighting which parameters need to be touched at all. Formally, we pose this as a relative $\epsilon$-constrained minimal-intervention problem. For a user-specified relative tolerance $\epsilon \in [0,1)$ on the achievable improvement over the default, the scalar terminal objective is
  \[
  \min_{\bm x \in \mathbb{X}} \|\bm x - \bm x^{\mathrm{def}}\|_0
  \quad \text{s.t.} \quad
  f(\bm x) \ge f^\star - \epsilon\,(f^\star - f(\bm x^{\mathrm{def}})),
  \qquad f^\star := \max_{\bm x' \in \mathbb{X}} f(\bm x').
  \]
  When $G_{\mathrm{def}}:=f^\star-f(\bm x^{\mathrm{def}})>0$, $\epsilon$ is the fraction of the default-to-optimum improvement we are willing to forgo for sparsity. If $G_{\mathrm{def}}=0$, the default is already globally optimal and is the zero-intervention solution. This defines an ideal \emph{terminal} target: the configuration with the fewest changed parameters that still achieves near-optimal true-function performance.\footnote{Our theory assumes $\mathbb X \subset \mathbb R^d$ is a compact hyper-rectangle; experiments also use mixed continuous--discrete spaces, where $\|\bm x-\bm x^{\mathrm{def}}\|_0$ counts coordinates of any type that differ from the default.}

  For a vector-valued maximization objective $\bm f=(f_1,\ldots,f_m):\mathbb X\to\mathbb R^m$ and a fixed hypervolume reference point $\bm z$, define the budgeted search space
  $\mathbb X_k:=\{\bm x\in\mathbb X:\|\bm x-\bm x^{\mathrm{def}}\|_0\le k\}$,
  the attainable nondominated objective set $\mathcal P_k^\star:=\operatorname{ND}(\bm f(\mathbb X_k))$, and its hypervolume $H_k^\star:=\operatorname{HV}_{\bm z}(\mathcal P_k^\star)$. The multi-objective terminal target uses the smallest intervention budget
  \[
  k_\epsilon^\star
  :=\min\bigl\{k\in\{0,\ldots,d\}:H_k^\star\ge H_d^\star-\epsilon(H_d^\star-H_0^\star)\bigr\}.
  \]
  If $H_d^\star=H_0^\star$, we set $k_\epsilon^\star=0$; otherwise $\epsilon$ has the same relative-improvement interpretation as above. For one objective and a reference point dominated by all relevant attainable values, this criterion is equivalent to the scalar formulation. A finite BO run can only approximate these population fronts with evaluated points, and our empirical sparse-frontier hypervolumes are finite-sample analogues of $H_k^\star$.

  \section{Methodology}
  \label{sec:bonsai}
  \subsection{Acquisition Gaps and Thresholded Pruning}
  \label{sec:af_gaps_and_thresholded_pruning}
  Solving either terminal objective directly is intractable because the true objective is unknown. BO instead acts on an acquisition function $\alpha_t$ at each iteration $t$. Let $\bm x_t^*$ denote the base candidate returned by acquisition optimization and define the \textit{acquisition gap} of any candidate $\bm x$ as $\Delta_t(\bm x) := \alpha_t(\bm x_t^*) - \alpha_t(\bm x)$. We impose a per-iteration relative threshold $\rho_t \in [0,1)$ on $\Delta_t$. Applying the relative threshold to raw $\alpha_t$ values is not well-posed when the acquisition baseline is arbitrary (e.g., EI vanishes far from data, qLogNEI lives on a log scale), so we instead apply it to a baseline-shifted incremental acquisition $\tilde\alpha_t(\bm x) := \alpha_t(\bm x) - b_t$, where $b_t := \max_{\bm x\in\mathcal D_{t-1}}\alpha_t(\bm x)$ is the maximum \emph{current-round} acquisition value over $\mathcal D_{t-1}$, the nonempty initial design together with all previously queried BONSAI designs (not the historical acquisition values recorded at query time). We require the base candidate to satisfy $\alpha_t(\bm x_t^*)\ge b_t$, so that $\tilde\alpha_t(\bm x_t^*)\ge0$. Exact maximization guarantees this because the previously evaluated designs are feasible candidates; an approximate solver can enforce it by comparing its output with the baseline-attaining previous design. This makes $\rho_t$ a fraction of the \emph{improvement over the current best} and renders it comparable across acquisition functions (for log-acquisitions \citep{logei} we exponentiate before forming $b_t$; see Appendix~\ref{appdx:bonsai_alg}). Our local per-iteration objective is
  $\min_{\bm x \in \mathbb{X}} \|\bm x - \bm x^{\mathrm{def}}\|_0 \quad \text{s.t.} \quad \Delta_t(\bm x) \le \rho_t \,\tilde\alpha_t(\bm x_t^*),$
  and Section~\ref{sec:theory} shows that controlling $\rho_t$ bounds cumulative regret, providing a principled proxy for the global $\epsilon$-objective.

  Given $S \subseteq \{1,\dots,d\}$, let $P_S(\bm x)$ be the configuration that keeps components in $S$ as in $\bm x$ and resets the rest to their defaults. BONSAI restricts to the family of pruned candidates $\mathcal X^{\mathrm{prune}} := \{ P_S(\bm x_t^*) : S \subseteq A(\bm x_t^*) \}$ and denotes the returned member by $\tilde{\bm x}_t$. Exact pruning selects, within this family and subject to the relative gap rule, a candidate with the smallest active set. Enumerating all $2^{|A(\bm x_t^*)|}$ subsets is infeasible in high dimensions, so BONSAI uses a greedy algorithm that resets components one at a time. Restricting to single-coordinate resets means BONSAI can miss configurations only reachable by joint resets, but on the low-dimensional problems that we test where enumeration is feasible, the two variants have comparable performance (Section~\ref{sec:experiments}).

  \subsection{Sequential Greedy Pruning}
  The pruned candidate $\tilde{\bm x}_t$ is initialized at the base acquisition candidate $\bm x_t^*$ in our practical sequential greedy variant (pseudocode in \cref{alg:bonsai}, Appendix~\ref{appdx:bonsai_alg}). BONSAI repeatedly resets the differing component $j$ whose pruned candidate has the smallest acquisition gap and still satisfies the relative threshold; it stops when no further reset is feasible. The result satisfies the relative-gap condition by construction. On the low-dimensional problems where exact enumeration is feasible, greedy and exact pruning have comparable performance and yield the same first post-initialization candidate in 85.7\% of runs (Section~\ref{sec:experiments}). In the worst case, BONSAI performs $O(d^2)$ acquisition evaluations per BO step (at most $d$ outer iterations, each over $|\mathcal D| \le d$ components), which is typically modest relative to $\alpha_t$-optimization or evaluating $f$.

  %%%%%%%%%%%%%%%%%%%%%%%%%%%%%%%%%%%%%%%%%%%%%%%%%%%%%%%%%%%%%%%%%%%%%%
  \section{Theoretical Analysis}
  \label{sec:theory}

  We analyze BONSAI in the sequential setting with GP-UCB; complete proofs are in Appendix~\ref{appdx:proofs}. Cumulative regret is $R_T = \sum_{t=1}^T [f(\bm x^\star) - f(\bm x_t)]$ and simple regret $r_T = f(\bm x^\star) - \max_{t \le T} f(\bm x_t) \le R_T/T$, with $\bm x^\star \in \argmax_{\bm x \in \mathbb X} f(\bm x)$. Under standard kernel regularity \citep{srinivas}, GP-UCB achieves sublinear cumulative regret at the rate $R_T/T = \tilde{O}(\gamma_T/\sqrt T)$, where $\gamma_T$ is the maximum information gain. \cref{sec:theory_preliminaries} gives the formal setup, and \cref{sec:theory_approx} expresses our bounds as this GP-UCB term plus an explicit, gap-dependent penalty.

  Our bounds depend only on the \emph{acquisition gaps} between queried points and the ideal UCB maximizers, not on how query points are constructed; BONSAI is one such construction that bounds these gaps while minimizing deviations from the default.

  %%%%%%%%%%%%%%%%%%%%%%%%%%%%%%%%%%%
  \subsection{Preliminaries}
  \label{sec:theory_preliminaries}

  We work in the standard GP-UCB setting of \citet{srinivas} and adopt the standard GP-UCB assumptions: $f$ lies in the reproducing kernel Hilbert space (RKHS) $\mathcal H_k$ associated with a positive definite kernel $k$, with bounded norm $\|f\|_{\mathcal H_k} \le B$, and the kernel variance is uniformly bounded so that $|f(\bm x)| \le \kappa B$ for all $\bm x \in \mathbb X$. We place a zero-mean GP prior with covariance $k$. Under these conditions, the GP posterior used at round $t$ has mean $\mu_{t-1}$ and standard deviation $\sigma_{t-1}$, and GP-UCB uses the acquisition function $\alpha^\text{UCB}_t$ defined in Section~\ref{sec:background}. As in the original GP-UCB analysis, we treat the kernel and noise hyperparameters as fixed and known when stating the guarantees; in practice they are typically learned from data.
  % , and our bounds should be interpreted as conditional on a particular choice of hyperparameters.

  Let $\bm x_t^{\mathrm{opt}} \in \arg\max_{\bm x\in\mathbb X}\alpha_t(\bm x)$ be an exact acquisition maximizer and $\alpha_t^{\star}:=\alpha_t(\bm x_t^{\mathrm{opt}})$. The main result below assumes exact base acquisition maximization, $\bm x_t^*=\bm x_t^{\mathrm{opt}}$. Appendix~\ref{appdx:proofs} gives the corresponding extension for an approximate base solver in terms of its additive acquisition error.

  To make BONSAI's relative rule well-posed, let $\mathcal D_{t-1}$ contain a nonempty initial design and all previously queried BONSAI designs. We use the incremental acquisition $\tilde\alpha_t(\bm x):=\alpha_t(\bm x)-b_t$ with $b_t:=\max_{\bm x\in\mathcal D_{t-1}}\alpha_t(\bm x)$ and apply thresholds relative to the base candidate. We require $\alpha_t(\bm x_t^*)\ge b_t$, so that $\tilde\alpha_t(\bm x_t^*)\ge0$; exact maximization guarantees this because previously queried designs are feasible acquisition candidates. Our goal is to characterize how the acquisition gap incurred by simplifying the base candidate affects cumulative and terminal performance.

  \subsection{Regret Bounds for BONSAI with GP-UCB}
  \label{sec:theory_approx}
  We assume that at each round $t$ the point $\tilde{\bm x}_t$ selected by the BO policy satisfies a gap rule relative to $\bm x_t^*$: $\Delta_t(\tilde{\bm x}_t) \le \rho_t \,\tilde\alpha_t(\bm x_t^*)$ for some user-specified sequence $(\rho_t)_{t\ge 1}$, where $\rho_t \in [0,1)$. Intuitively, $\rho_t$ is the fraction of incremental acquisition value we are willing to sacrifice at time $t$ in order to simplify the configuration. A direct GP-UCB confidence-event argument yields the following regret bound.
  \begin{restatable}[Regret bound under relative pruning]{theorem}{klcspecialization}
  \label{thm:klc-specialization}
  Assume the standard GP-UCB confidence and information-gain conditions: $f\in\mathcal H_k$ with $\|f\|_{\mathcal H_k}\le B$, $k(x,x)\le 1$, and the noise is conditionally $R$-sub-Gaussian. Assume an initial design of fixed size $n_0\ge1$ has been observed before round $t=1$; $T$ and $R_T$ count the subsequent $T$ BONSAI rounds. In this post-initialization indexing, $\mu_{t-1}$ and $\sigma_{t-1}$ denote the posterior after $n_0+t-1$ total observations, and $\gamma_t$ denotes the usual maximum information gain after $n_0+t$ total evaluations. Let $\alpha_t(\bm x) = \mu_{t-1}(\bm x) + \beta_t \sigma_{t-1}(\bm x)$ be the GP-UCB acquisition with
  $\beta_t = B + R\sqrt{2\bigl(\gamma_{t-1} + 1 + \log(1/\delta)\bigr)},$
  where $\delta \in (0,1]$.
  Assume the base acquisition solver is exact, $\bm x_t^*=\bm x_t^{\mathrm{opt}}$, and BONSAI returns $\tilde{\bm x}_t$ satisfying $\Delta_t(\tilde{\bm x}_t)\le\rho_t\tilde\alpha_t(\bm x_t^*)$, where $\rho_t\in[0,1)$. Then, with probability at least $1-\delta$,
  \begin{equation}
  \label{eq:bonsai-klc}
    R_T
    = O\Bigl(\gamma_T\sqrt{T} \;+\; \sqrt{\gamma_T}\,\sum_{t=1}^T \rho_t\Bigr).
  \end{equation}
  \end{restatable}
  The bound is the standard GP-UCB rate plus an additive penalty proportional to $\sum_t\rho_t$. It immediately yields the simple-regret bound
  \[
  f^\star-\max_{t\le T}f(\tilde{\bm x}_t)
  \le \frac{R_T}{T}
  =O\!\left(\frac{\gamma_T}{\sqrt T}+\frac{\sqrt{\gamma_T}}{T}\sum_{t=1}^T\rho_t\right).
  \]
  \begin{restatable}[Asymptotic No-Regret]{corollary}{bonsaiklc}
  \label{cor:bonsai-klc}
  If $\gamma_T=o(\sqrt T)$ and $\sqrt{\gamma_T}\sum_{t=1}^T\rho_t=o(T)$, then $R_T/T\to0$. In particular, $\sum_{t=1}^T\rho_t=O(\sqrt T)$ is sufficient under the stated information-gain condition and preserves the standard GP-UCB term up to no larger order.
  \end{restatable}
  For example, $\rho_t=c/(t+1)$ with $0<c<1$ gives $\sum_{t=1}^T\rho_t=O(\log T)$ while respecting $\rho_t<1$ (see \cref{cor:schedules} in Appendix~\ref{appdx:theory}).

  \paragraph{Sparsity recovery guarantees.}
  The regret result above conditions on fixed known kernel hyperparameters. Our support result makes the stronger, idealized structural assumption that every strictly irrelevant coordinate has infinite ARD lengthscale; an additional relevant-dimension gap prevents pruning modified relevant coordinates (full statements in Appendix~\ref{appdx:sparsity_recovery}):
  \begin{itemize}[leftmargin=14pt, labelwidth=!, labelindent=0pt, itemsep=1pt, topsep=2pt]
      \item \textbf{(Perfect ARD separation, A.\ref{assum:perfect_ard}).} $\ell_j=\infty$ for every irrelevant coordinate $j\notin A^\star$, and $\ell_j<\infty$ for $j\in A^\star$.
      \item \textbf{(Relevant-dimension gap, A.\ref{assum:relevant_gap}).} At each round, resetting any modified relevant coordinate after the irrelevant resets costs more than the pruning threshold.
  \end{itemize}
  Perfect ARD alone lets BONSAI prune every irrelevant deviation from the base candidate at zero acquisition cost, so each returned candidate is supported only on truly relevant coordinates. With the relevant-gap assumption, BONSAI additionally retains every relevant deviation proposed by the base candidate. The following proposition combines this sparsity guarantee with global regret and terminal performance guarantees.

  \begin{proposition}[Relevant-coordinate support and global performance feasibility]
  \label{prop:regret_plus_sparsity}
  Suppose the GP-UCB regret assumptions, exact base acquisition maximization, and Assumption~\ref{assum:perfect_ard} hold. At every round, assume either $\rho_t=0$ or Assumption~\ref{assum:relevant_gap} holds. Then:
  \begin{enumerate}[leftmargin=14pt, labelwidth=!, labelindent=0pt, itemsep=2pt, topsep=2pt]
      \item \emph{Relevant-coordinate support.} BONSAI prunes every modified irrelevant coordinate and returns $A(\tilde{\bm x}_t)\subseteq A^\star:=\{j:\ell_j<\infty\}$. Under Assumption~\ref{assum:relevant_gap}, it also retains every modified relevant coordinate, so $A(\tilde{\bm x}_t)=A(\bm x_t^*)\cap A^\star$.
      \item \emph{Global regret.} The returned query incurs $\Delta_t(\tilde{\bm x}_t)=0$, so $R_T=O(\gamma_T\sqrt T)$ without the threshold-dependent penalty in Theorem~\ref{thm:klc-specialization}.
      \item \emph{Terminal performance feasibility.} Let $\bm x_T^{\mathrm{best}}\in\arg\max_{t\le T}f(\tilde{\bm x}_t)$ be the oracle best queried point. Its simple regret is $O(\gamma_T/\sqrt T)$. If $G_{\mathrm{def}}>0$, then it is feasible for the scalar terminal problem with $\epsilon_T=O(\gamma_T/\sqrt T)/G_{\mathrm{def}}$ whenever $\epsilon_T<1$, and $\|\bm x_T^{\mathrm{best}}-\bm x^{\mathrm{def}}\|_0\le|A^\star|$. If $G_{\mathrm{def}}=0$, the default itself is the zero-intervention terminal solution. This establishes global performance feasibility with relevant-coordinate support.
  \end{enumerate}
  \end{proposition}

  For finite lengthscales, Appendix~\ref{appdx:sparsity} gives the conditional reset bound
  $\Delta_t(R_S(\bm x_t^*))\le\sum_{j\in S}C_tR_{t,j}/\ell_j$ under stated regularity assumptions. Any reset set whose bound is at most $\rho_t\tilde\alpha_t(\bm x_t^*)$ is feasible and has correspondingly bounded acquisition loss. This does not guarantee that sequential greedy pruning selects that set; its contribution to regret is controlled only when the normalized per-round losses have a suitably controlled cumulative sum.

  The infinite-lengthscale assumption is an idealization: every experiment estimates finite lengthscales from data. Figures~\ref{fig:synthetic} and~\ref{fig:real} validate the downstream sparsity--performance behavior motivating the result. On synthetic embedded problems, BONSAI reaches near-optimal performance near the known intrinsic dimensionality; on real-world tasks, where the true active set is unknown, it produces substantially sparser high-performing recommendations. Appendix~\ref{appdx:lengthscale_est} provides a separate finite-sample lengthscale-estimation study on DTLZ2.

  % \subsection{Relative Thresholds}
  % \label{sec:thresholds-ui}

  % In practice, it is often easier to reason about thresholds in relative terms, for example by specifying a tolerance $\rho_t \in [0,1)$ such that
  % \[
  % \Delta_t(\tilde{\bm x}_t) \le \rho_t \,\alpha_t(\bm x_t^*).
  % \]
  % Under the GP-UCB assumptions, the acquisition values at the maximizers are uniformly bounded as
  % \[
  % \alpha_t(\bm x_t^*) \le C_t := \kappa B + 2\kappa\sqrt{\beta_t}
  % \quad\text{for all }t,
  % \]
  % see Lemma~\ref{lemma:alpha_bound} in Appendix~\ref{appdx:theory}. Thus any relative rule implies an absolute gap budget with
  % \[
  % \tau_t = \rho_t C_t
  % \quad\Rightarrow\quad
  % \Delta_t(\tilde{\bm x}_t) \le \tau_t,
  % \]
  % and Theorem~\ref{thm:abs_regret} applies directly with these $\tau_t$. For the rest of the paper we therefore use the absolute formulation in all statements and proofs, and view relative thresholds as a convenient reparameterization for practitioners. An explicit regret bound written in terms of $(\rho_t)$ is given in Corollary~\ref{thm:rel_regret} in Appendix~\ref{appdx:theory}.
  % \sd{We should discuss how we always use incremental AFs, which makes setting rho easier.}

  %%%%%%%%%%%%%%%%%%%%%%%%%%%%%%%%%%%

  %%%%%%%%%%%%%%%%%%%%%%%%%%%%%%%%%%%
  \subsection{Extensions}
  \textbf{Batch candidate generation.}
  In settings where multiple points are evaluated in parallel at each BO iteration, one typically maximizes a batch acquisition function $\alpha : \mathbb X^q \to \mathbb R$ over tuples $X = (\bm x_1,\dots,\bm x_q)$. BONSAI can be extended to this setting by applying the pruning procedure to each point in the batch while conditioning on the previously selected pruned batch entries. In the batch setting, we use the incremental acquisition function described in Section~\ref{sec:af_gaps_and_thresholded_pruning} when pruning the first point. For pruning subsequent points, we use the incremental improvement over the previously selected and pruned points in the batch.

  \textbf{Acquisition functions.}
  Our theoretical results are specific to the GP-UCB acquisition.
  BONSAI itself, however, is an acquisition-conditioned candidate transformation that only requires the ability to evaluate an acquisition function at candidate points. It changes the point evaluated by BO and therefore influences subsequent observations and model fits. In particular, it can be used with EI and its variants, which we find to be more robust in some of our empirical studies. In those cases, the regret guarantees in \cref{sec:theory} should be viewed as a guide rather than a formal guarantee: they capture what happens if one could construct an acquisition function with analogous confidence properties and bounded gaps.

  %%%%%%%%%%%%%%%%%%%%%%%%%%%%%%%%%%%%%%%%%%%%%%%%%%%%%%%%%%%%%%%%%%%%%%
  \section{Experiments}
  \label{sec:experiments}
  %We evaluate BONSAI on a variety of synthetic and real-world benchmark problems.
  %We compare against the following baselines: quasi-random Sobol sampling, vanilla BO, and two methods for sparse BO, IR and ER \citep{sebo}.
  We evaluate BONSAI on a variety of synthetic and real-world benchmark problems. We compare against quasi-random Sobol sampling, standard BO (EI), three sparse-BO methods from \citet{sebo} (IR, ER, and SEBO), and prior-guided Bayesian optimization (PiBO; \citealp{hvarfner2022pibo}). For IR/ER, we use an $\ell_0$ penalty coefficient of $0.01$, which is reported in SEBO as a strong default choice.\footnote{IR and ER were proposed in \citet{sebo} for single-objective optimization, but we extend both to constrained optimization, and we extend ER to multi-objective optimization. See Appendix~\ref{appdx:ir_er_exts} for further discussion.} See Appendix~\ref{appdx:ir_er_sens} for a sensitivity analysis of IR and ER. All BO methods use the same ensemble MAP-SAAS GP surrogate: a Mat\'ern-5/2 GP with a SAAS prior \citep{saasbo}, whose hyperparameters are fit by sampling global sparsity scales and performing MAP estimation (Appendix~\ref{appdx:map_saas}).

  We initialize each method with 20 trials total: the default point $\bm x^\text{def}$ followed by 19 Sobol samples. We run 20 independent replications. For PiBO, we center a product Laplace prior at the same default in normalized $[0,1]^d$ coordinates, use scale $b=0.1$, and multiply the acquisition by $\pi(\bm x)^{\beta/n}$ at the $n$th BO iteration, with $\beta=N/10$ and $N$ the post-initialization BO budget. See Appendix~\ref{appdx:empirical} for the full specification.

  We leverage qLogNEI and qLogNEHVI \citep{logei} as the EI-based acquisition functions for single- and multi-objective problems, respectively. For BONSAI with EI, we use a fixed relative threshold $\rho=0.2$---matching the configuration whose sensitivity is analyzed in Appendix~\ref{appdx:bonsai_rho_sens}---while for BONSAI with UCB we use $\rho_t=1/t$. This has the same asymptotic behavior as the schedule in Corollary~\ref{cor:schedules}. To apply UCB to multi-objective problems, we use random hypervolume scalarizations \citep{pmlr-v119-zhang20i}. Appendix~\ref{appdx:bonsai_ucb} shows EI and UCB perform comparably on most problems. All methods are implemented using Ax~\citep{olson2025ax} with components from BoTorch~\citep{balandat2020botorch}; BONSAI code is available at \url{https://ax.dev/}.

  % We briefly outline the main empirical questions that our experiments are designed to address and the evaluation protocol we use. Detailed experimental results, figures, and tables follow the structure described in this section.

  % \paragraph{Experimental questions.}
  % Our experiments are organized around the following questions:
  % \begin{enumerate}
  %     \item Does BONSAI significantly reduce the number of parameters that differ from the default configuration in the recommendations produced by BO?
  %     \item How does BONSAI affect optimization performance compared to standard BO, as measured by regret or best observed objective value?
  %     \item How sensitive is BONSAI to the choice of thresholds and to the choice of acquisition function (UCB versus EI)?
  %     \item How close is the sequential greedy pruning algorithm in \cref{alg:bonsai} to the exact combinatorial solution in low-dimensional problems where exact enumeration is feasible?
  %     \item How does BONSAI affect candidate generation time compared to alternatives?
  % \end{enumerate}
  \subsection{Benchmarks}
  Our test problems include single-objective, constrained, and multi-objective settings. All problems are maximization problems--we negate the objective for minimization problems. We report optimization performance (best observed feasible objective for constrained problems and hypervolume (HV) for multi-objective problems) and candidate generation time--acquisition optimization and pruning (for BONSAI)--in sequential ($q=1$) settings. See Appendix~\ref{appdx:batch_optimization} for results in batch ($q=5$) settings.

  In our default-aware benchmarks, the supplied default represents a status quo from which changes are measured. For every benchmark, PiBO's prior center is exactly the same input-space target $\bm x^\text{def}$ that BONSAI prunes toward. Our ranking benchmarks (Appendix~\ref{appdx:additional_ranking_problems}) use the control configuration from real-world A/B tests, LassoDNA uses zero penalty, and the synthetic embeddings use the center of the search space. Appendix~\ref{appdx:empirical} gives further construction and availability details.

  \textbf{Synthetic.} Branin (2d), Hartmann (6d), PressureVessel (4d, 4 constraints)~\citep{pressure_vessel}, and DTLZ2 (6d, 2-objective)~\citep{dtlz}, each embedded in a 50d search space by adding irrelevant parameters.

  \textbf{Real-world.} (i) \emph{Joint NAS/HPO} (26d, mixed integer/continuous): a GP surrogate of a multi-task, multi-label neural network architecture/HPO problem at a large web service. (ii) \emph{Optical Design} (146d, 2-objective)~\citep{morbo}: surface morphology and geometry of an AR see-through display, evaluated via a neural network surrogate of physical simulations. (iii) \emph{LassoDNA} (180d)~\citep{2022lassobench}: per-feature weighted-Lasso penalties on the DNA dataset~\citep{libsvm}. (iv) \emph{Cell Network} (30d, mixed, 2-objective)~\citep{dreifuerst2021}: per-antenna transmission power and downtilt, minimizing under- and over-coverage. See Appendix~\ref{appdx:empirical} for full descriptions and additional benchmarks (including noisy real-world ranking problems).

  %%%%%%%%%%%%%%%%%%%%%%%%%%%%%%%%%%%
  \subsection{Results}
  \label{subsec:experiments:results}

  \begin{figure*}[!t]
    \centering
    \includegraphics[width=\linewidth]{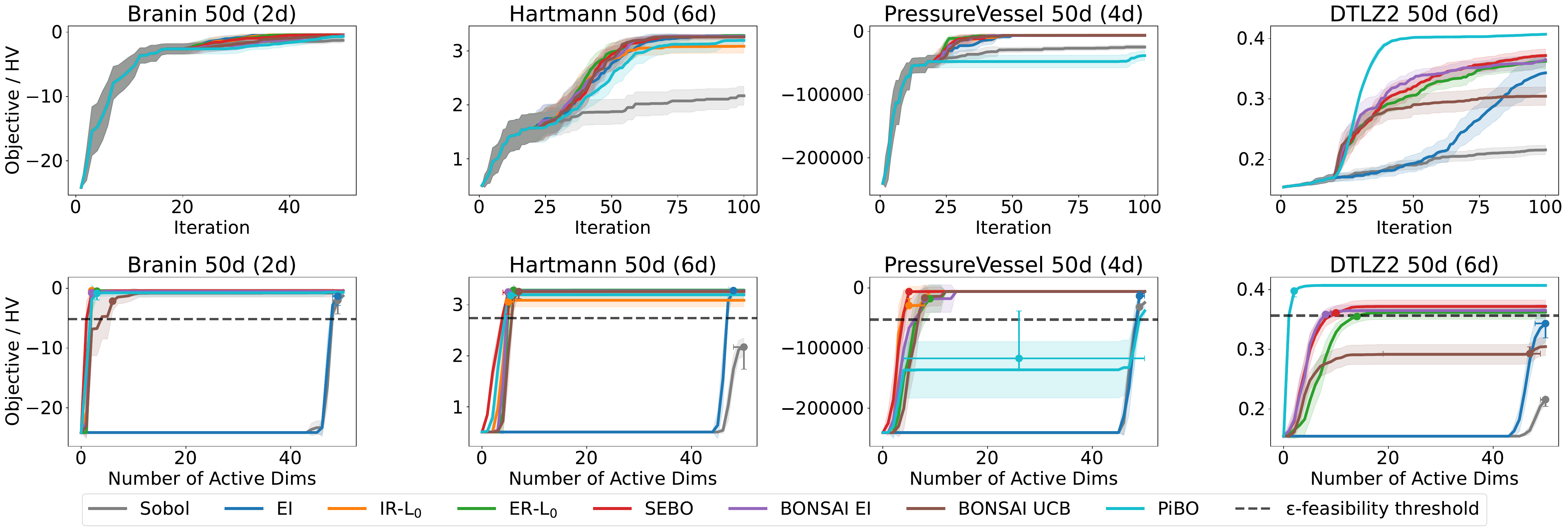}
    \caption{Synthetic 50-dimensional embeddings. Top: Objective or HV. Bottom: Best Objective (or HV) value for each level of active dimensions. For MOO, the HV at sparsity level $k$ is the hypervolume of the feasible Pareto frontier computed over all evaluated points with at most $k$ active dimensions.}
    \label{fig:synthetic}
    \vspace{0.5em}
    \includegraphics[width=\linewidth]{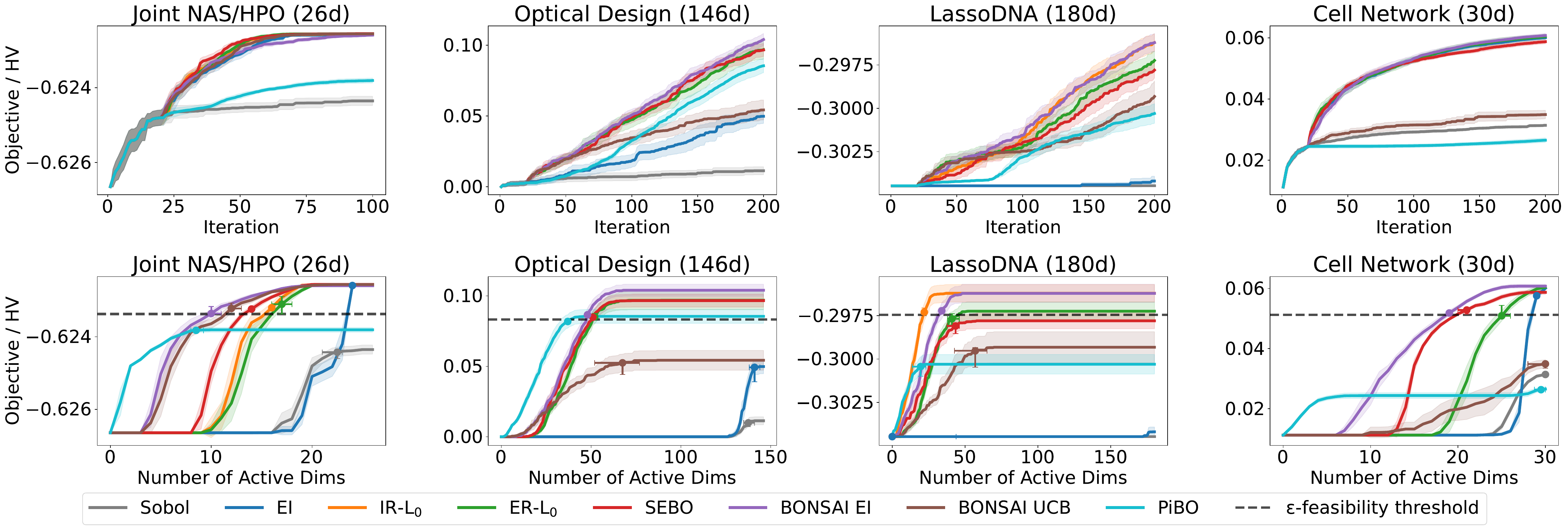}
    \caption{Real-world problems. Top: Objective or HV. Bottom: Best Objective (or HV) value for each level of active dimensions. For MOO, the HV at sparsity level $k$ is the hypervolume of the feasible Pareto frontier computed over all evaluated points with at most $k$ active dimensions.}
    \label{fig:real}
  \end{figure*}

  We report optimization performance over the number of evaluations, the best objective found for a given number of active dimensions, and candidate generation time. For multi-objective problems, we report, for each sparsity level $k=0,\dots,D$, the hypervolume of the feasible Pareto frontier formed by all evaluated points with at most $k$ active dimensions (i.e., $\|\bm x - \bm x^{\mathrm{def}}\|_0 \le k$). Curves show the mean and $\pm$ 2 standard errors across 20 replications.

  In the bottom row of Figures~\ref{fig:synthetic} and~\ref{fig:real}, we overlay one marker per method summarizing the per-replication \emph{minimal-intervention points} at $\epsilon=0.2$. For each method configuration we average its per-replication best value, define $f_{\max}$ as the largest such mean, and let $f_{\mathrm{def}}$ be the mean default value across the stored results. We then use the threshold $f_{\max}-0.2(f_{\max}-f_{\mathrm{def}})$; for MOO we use the analogous benchmark-wide HV threshold (Appendix~\ref{appdx:moo_marker}). For each replication, we select the sparsest explored configuration satisfying this relative $\epsilon$-constraint; if none does, we use its best sparse-frontier point as a plotting fallback. Since these selected points are bounded and skewed, each is first snapped to the closest point on the method's mean sparse frontier in normalized sparsity/objective coordinates, and we plot the median across replications with 25th--75th percentile error bars in both sparsity and objective. These markers are empirical finite-run analogues of the terminal objectives in Section~\ref{sec:problem_setting}: they use the best mean achieved by the tested methods as a proxy for the unknown population optimum.

  Across all benchmarks, we find that BONSAI yields favorable sparsity--performance trade-offs: it produces substantially simpler recommendations while maintaining competitive objective or hypervolume performance, and it is typically the fastest default-aware method with respect to generation time. Its minimal-intervention marker generally lies toward the lower-left of the alternatives, while the IQR bars and separately reported feasibility counts show the variation across replications; fallback points remain classified as infeasible. SEBO produces sparse solutions but at significantly higher wall-time cost than BONSAI, with no consistent advantage in sparsity--quality trade-off and substantially worse performance on mixed-discrete problems.
  On synthetic benchmarks in Figure~\ref{fig:synthetic}, BONSAI recovers near-optimal solutions at or close to the true number of relevant parameters; standard BO achieves strong objective values via dense changes, while BONSAI yields better objective/HV at comparable sparsity. IR and ER also find near-optimal sparse solutions, but IR fails on Hartmann 50d, and BONSAI outperforms ER on DTLZ2 50d. PiBO performs best on DTLZ2 50d due to its  sparsity-inducing prior heavily weighting the region of the search space containing the Pareto set.
  % On synthetic benchmarks in Figure~\ref{fig:synthetic}, we find that BONSAI recovers optimal solutions with very few active parameters--at or close to the true number of relevant parameters--whereas vanilla BO does not find near-optimal sparse solutions. IR and ER also find near-optimal sparse solutions, but IR fails to find the optimal solution on the Hartmann 50d problem. IR and ER are focused on the single objective case \citep{sebo}, not multi-objective. We extend IR and ER to the constrained setting--see Appendix~\ref{appdx:ir_er_exts} for further discussion.

  On the real problems in Figure~\ref{fig:real}, BONSAI likewise finds sparse high-performing configurations; for example, on Optical Design it finds excellent Pareto frontiers using an average of fewer than 60 active dimensions ($<50\%$ of the parameters). Because these tasks do not provide a ground-truth active set, this demonstrates operational sparsity rather than exact recovery of irrelevant coordinates. BONSAI outperforms standard BO on DTLZ2, Optical Design, and LassoDNA. Improved lengthscale estimation is one possible mechanism, which we examine directly only for DTLZ2 in Figure~\ref{fig:lengthscale_est} (Appendix~\ref{appdx:lengthscale_est}). PiBO produces no $\epsilon$-feasible point in any of 20 replications on Joint NAS/HPO, LassoDNA, or Cell Network, and succeeds in 14/20 replications on Optical Design. BONSAI succeeds in 20/20, 16/20, 20/20, and 20/20 replications on Joint NAS/HPO, LassoDNA, Cell Network, and Optical Design, respectively. Thus PiBO's soft prior can bias search toward the default without reliably recovering a high-performing sparse point, whereas BONSAI's acquisition-gap pruning preserves optimization quality more consistently.

  BONSAI's candidate generation time is on-par with standard BO, while IR/ER/SEBO are substantially slower (Appendix Table~\ref{table:gen_times_bonsai_map_saas_problems_combined}; see Appendix~\ref{appdx:gen_times} for additional notes on per-problem wall-time anomalies, omitted batch runs, and the active-dimension tolerance used for continuous-relaxation baselines). On lower-dimensional benchmarks where exact pruning is feasible, all methods are evaluated and BONSAI's optimization performance is comparable to Vanilla BO with EI and to the default-aware baselines (Figure~\ref{fig:bonsai_exact_main}; full results in Appendix~\ref{appdx:bonsai_exact}).

  \begin{figure*}[!t]
    \centering
    \includegraphics[width=\linewidth]{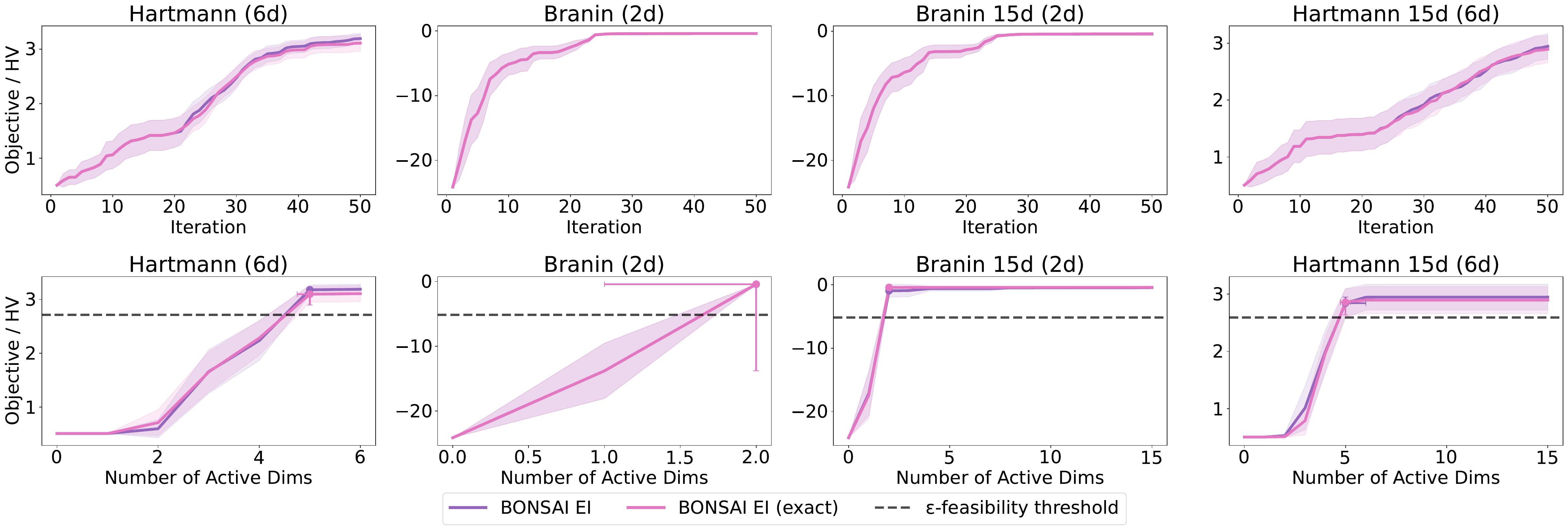}
    \caption{Methods evaluated in the low-dimensional exact-pruning study (Sobol, Vanilla BO, IR, ER, SEBO, and BONSAI EI with sequential greedy and exact pruning). Exact pruning costs $O(2^{|A(\bm x_t^*)|})$. BONSAI's greedy and exact variants yield comparable optimization performance, and BONSAI is competitive with the evaluated default-aware baselines. On Branin (2d), the BONSAI greedy and exact traces overlap exactly. Top: objective. Bottom: best objective for each number of active dimensions.}
    \label{fig:bonsai_exact_main}
  \end{figure*}

  \textbf{Greedy versus exact pruning.}
  Greedy may miss joint-reset configurations. Empirically, on every low-dimensional problem where exact $O(2^{|A(\bm x_t^*)|})$ enumeration is feasible, greedy and exact pruning yield comparable performance and return the same first post-initialization candidate 85.7\% of the time (Appendix~\ref{appdx:bonsai_exact}). This is not a trajectory-wide equivalence guarantee, and exact pruning is infeasible at the scales of our main experiments.
  % We measure how often the greedy solution matches the exact minimum-$\ell_0$ solution under the gap constraints and how large the discrepancy is when they differ. This comparison provides empirical evidence on how often the simple greedy path coincides with, or closely approximates, the ideal two-stage solution. \sd{TODO: add figure for this.}

  %%%%%%%%%%%%%%%%%%%%%%%%%%%%%%%%%%%%%%%%%%%%%%%%%%%%%%%%%%%%%%%%%%%%%%
  \section{Discussion}
  \label{sec:discussion}

  BONSAI is a lightweight, default-aware BO policy that promotes operational simplicity by minimizing deviations from a default subject to an acquisition-gap constraint. The resulting recommendations are easier to inspect, test, and deploy than dense alternatives. For GP-UCB, BONSAI can be viewed as inexact acquisition maximization \citep{pmlr-v286-kim25b}, with the standard cumulative-regret term plus an explicit penalty controlled by the cumulative pruning thresholds. Coordinate-wise $\ell_0$ is most natural for independently actionable knobs; structured, grouped, or weighted notions may better reflect other applications.

  A central limitation is that sparsity is defined relative to the supplied default. A centered default can make boundary-seeking dense candidates appear especially complex, while a boundary default can make some candidates appear sparse by coincidence. BONSAI still tests each reset through its model- and acquisition-conditioned contribution, but the observed sparsity and comparative advantage can depend on default placement. Additional limitations include the idealized perfect-ARD and relevant-gap assumptions, the fact that pruning all irrelevant deviations does not establish minimum cardinality among all terminally feasible configurations, finite-lengthscale bounds that require regularity and do not alone guarantee a greedy path, the single-coordinate greedy procedure, and theory specific to GP-UCB rather than EI. Appendix~\ref{appdx:future_work} discusses these directions further.

  % BONSAI provides a simple mechanism for making BO recommendations more conservative and interpretable in settings where there is a meaningful default configuration. By explicitly minimizing the number of deviations from the default subject to a cap on the acquisition gap, BONSAI encourages solutions that are easier to inspect, test, and deploy than unconstrained BO suggestions, and helps avoid issues caused by detrimental effects on the system that are not fully captured by the optimization objective.

  %%%%%%%%%%%%%%%%%%%%%%%%%%%%%%%%%%%%%%%%%%%%%%%%%%%%%%%%%%%%%%%%%%%%%%

  \bibliography{icml_2026}
  \bibliographystyle{plainnat}

  %%%%%%%%%%%%%%%%%%%%%%%%%%%%%%%%%%%%%%%%%%%%%%%%%%%%%%%%%%%%%%%%%%%%%%%%%%%%%%%
  %%%%%%%%%%%%%%%%%%%%%%%%%%%%%%%%%%%%%%%%%%%%%%%%%%%%%%%%%%%%%%%%%%%%%%%%%%%%%%%
  % APPENDIX
  %%%%%%%%%%%%%%%%%%%%%%%%%%%%%%%%%%%%%%%%%%%%%%%%%%%%%%%%%%%%%%%%%%%%%%%%%%%%%%%
  %%%%%%%%%%%%%%%%%%%%%%%%%%%%%%%%%%%%%%%%%%%%%%%%%%%%%%%%%%%%%%%%%%%%%%%%%%%%%%%
  \newpage
  \appendix
  \section{Proofs and Additional Theoretical Results}
  \label{appdx:theory}
  \subsection{Proofs}
  \label{appdx:proofs}
  \klcspecialization*
  \begin{proof}
  Work on the standard GP-UCB confidence event, which holds with probability at least $1-\delta$ and on which
  \[
  \mu_{t-1}(\bm x)-\beta_t\sigma_{t-1}(\bm x)
  \le f(\bm x)\le
  \mu_{t-1}(\bm x)+\beta_t\sigma_{t-1}(\bm x)
  \]
  for every $t$ and $\bm x\in\mathbb X$. Let $\bm x^\star\in\arg\max_{\bm x\in\mathbb X}f(\bm x)$. The upper confidence bound at $\bm x^\star$ and exact acquisition maximization give
  \[
  f^\star=f(\bm x^\star)
  \le\alpha_t(\bm x^\star)
  \le\alpha_t(\bm x_t^*).
  \]
  The lower confidence bound at the queried point gives
  \[
  f(\tilde{\bm x}_t)
  \ge\mu_{t-1}(\tilde{\bm x}_t)
      -\beta_t\sigma_{t-1}(\tilde{\bm x}_t).
  \]
  Combining these inequalities and then using the definitions of UCB and the acquisition gap yields
  \begin{align*}
  f^\star-f(\tilde{\bm x}_t)
  &\le \alpha_t(\bm x_t^*)
      -\mu_{t-1}(\tilde{\bm x}_t)
      +\beta_t\sigma_{t-1}(\tilde{\bm x}_t)\\
  &=\alpha_t(\bm x_t^*)
      -\bigl[\mu_{t-1}(\tilde{\bm x}_t)
             +\beta_t\sigma_{t-1}(\tilde{\bm x}_t)\bigr]
      +2\beta_t\sigma_{t-1}(\tilde{\bm x}_t)\\
  &=\alpha_t(\bm x_t^*)-\alpha_t(\tilde{\bm x}_t)
      +2\beta_t\sigma_{t-1}(\tilde{\bm x}_t)\\
  &=\Delta_t(\tilde{\bm x}_t)
      +2\beta_t\sigma_{t-1}(\tilde{\bm x}_t).
  \end{align*}

  We next bound the scale of the incremental acquisition. The confidence event, $\|f\|_{\mathcal H_k}\le B$, and $k(\bm x,\bm x)\le1$ imply $|f(\bm x)|\le B$ and
  \begin{align*}
  \alpha_t(\bm x_t^*)
  &=\mu_{t-1}(\bm x_t^*)+\beta_t\sigma_{t-1}(\bm x_t^*)\\
  &\le f(\bm x_t^*)+2\beta_t\sigma_{t-1}(\bm x_t^*)\\
  &\le B+2\beta_t,
  \end{align*}
  where the last inequality also uses $\sigma_{t-1}(\bm x)\le1$. Conversely, for every $\bm x\in\mathcal D_{t-1}$,
  \[
  \alpha_t(\bm x)
  =\mu_{t-1}(\bm x)+\beta_t\sigma_{t-1}(\bm x)
  \ge f(\bm x)\ge-B.
  \]
  Because $\mathcal D_{t-1}$ is nonempty, $b_t=\max_{\bm x\in\mathcal D_{t-1}}\alpha_t(\bm x)\ge-B$. Hence
  \[
  \tilde\alpha_t(\bm x_t^*)
  =\alpha_t(\bm x_t^*)-b_t
  \le (B+2\beta_t)-(-B)
  =2(B+\beta_t).
  \]
  Substituting this into the relative rule gives
  \[
  \Delta_t(\tilde{\bm x}_t)
  \le\rho_t\tilde\alpha_t(\bm x_t^*)
  \le2\rho_t(B+\beta_t).
  \]
  Therefore the instantaneous regret satisfies
  \[
  f^\star-f(\tilde{\bm x}_t)
  \le2\rho_t(B+\beta_t)
      +2\beta_t\sigma_{t-1}(\tilde{\bm x}_t).
  \]
  Summing over $t$ and using that $\beta_t$ is nondecreasing yields
  \[
  R_T\le2(B+\beta_T)\sum_{t=1}^T\rho_t
      +2\beta_T\sum_{t=1}^T\sigma_{t-1}(\tilde{\bm x}_t).
  \]
  Finally, Cauchy--Schwarz and the standard information-gain inequality give
  \[
  \sum_{t=1}^T\sigma_{t-1}(\tilde{\bm x}_t)
  \le\sqrt{T\sum_{t=1}^T\sigma_{t-1}^2(\tilde{\bm x}_t)}
  =O(\sqrt{T\gamma_T}).
  \]
  Under the stated exploration schedule, $\beta_T=O(\sqrt{\gamma_T})$ after absorbing fixed $B$, $R$, and $\delta$ into the constants. Consequently,
  \[
  R_T
  =O\!\left(\sqrt{\gamma_T}\sum_{t=1}^T\rho_t\right)
   +O(\gamma_T\sqrt T),
  \]
  which is~\eqref{eq:bonsai-klc}.
  \end{proof}

  \paragraph{Approximate base acquisition optimization.}
  If the base solver is approximate, define its additive acquisition error by
  \[
  e_t^{\mathrm{solve}}:=\alpha_t(\bm x_t^{\mathrm{opt}})-\alpha_t(\bm x_t^*)\ge0
  \]
  and require $\alpha_t(\bm x_t^*)\ge b_t$. The same argument gives
  \[
  R_T=O\!\left(\gamma_T\sqrt T+\sqrt{\gamma_T}\sum_{t=1}^T\rho_t+\sum_{t=1}^T e_t^{\mathrm{solve}}\right).
  \]
  A multiplicative solver guarantee can be recovered after fixing a particular nonnegative scale. On the confidence event, $\bar\alpha_t:=\alpha_t+B$ is nonnegative. If $\bar\alpha_t(\bm x_t^*)\ge\eta_t^{\mathrm{solve}}\bar\alpha_t(\bm x_t^{\mathrm{opt}})$, then $e_t^{\mathrm{solve}}=O(\sqrt{\gamma_T}(1-\eta_t^{\mathrm{solve}}))$, giving the corresponding additional term $O(\sqrt{\gamma_T}\sum_{t=1}^T(1-\eta_t^{\mathrm{solve}}))$.
  \bonsaiklc*
  \begin{proof}
  Dividing~\eqref{eq:bonsai-klc} by $T$ gives $R_T/T=O(\gamma_T/\sqrt T+\sqrt{\gamma_T}\sum_{t=1}^T\rho_t/T)$. Both terms vanish under the stated conditions. If $\sum_{t=1}^T\rho_t=O(\sqrt T)$, the threshold-dependent term is no larger in order than the standard term.
  \end{proof}
  \begin{restatable}[Example schedules]{corollary}{bonsaischedules}
  \label{cor:schedules}
  Assume $\gamma_T=O(\log^{d+1}T)$ for the SE kernel or $\gamma_T=O\bigl(T^{d/(d+2\nu)}\log^{2\nu/(2\nu+d)}T\bigr)$ for a Mat\'ern kernel, as in \citet{pmlr-v130-vakili21a}. Let $0<c<1$ and $\epsilon>0$.
  If $\rho_t=c/(t+1)^{1+\epsilon}$, then $\sum_{t=1}^T\rho_t=O(1)$ and $R_T=O(\gamma_T\sqrt T)$. If $\rho_t=c/(t+1)$, then
  \[
  R_T=O\Bigl(\gamma_T\sqrt T+\sqrt{\gamma_T}\log T\Bigr).
  \]
  These schedules are asymptotically no-regret whenever $\gamma_T=o(\sqrt T)$, including the SE bound and the displayed Mat\'ern bound when $\nu>d/2$.
  \end{restatable}
  \begin{proof}
  The first schedule has a finite sum, while the second satisfies $\sum_{t=1}^T\rho_t=O(\log T)$. Substitution in~\eqref{eq:bonsai-klc} gives the displayed rates. Dividing by $T$ and using $\gamma_T=o(\sqrt T)$ proves the no-regret statement.
  \end{proof}
  \subsection{Thresholds and Acquisition Gaps: A Toy Additive Model}
  \label{appdx:sparsity}
  The regret bounds above characterize the optimization cost of allowing acquisition gaps, but they do not by themselves explain how the thresholds influence the \emph{sparsity} of BONSAI's recommendations. We now give a brief, concrete picture of this interaction for a broad class of ARD kernels that includes both SE-ARD and Matérn-ARD (e.g., Matérn-$5/2$). This subsection helps explain which coordinates BONSAI tends to prune under common ARD kernels.

  We assume an ARD kernel of the form
  \[
      k(\bm x, \bm x') = \sigma_f^2 \,\phi\!\left(
          r(\bm x, \bm x')
      \right),
      \enspace
      r(\bm x, \bm x') = \sqrt{\sum_{j=1}^d \frac{(x_j - x_j')^2}{\ell_j^2}},
  \]
  where $\ell_j>0$ are per-component lengthscales and $\phi:[0,\infty)\to\mathbb R$ is a continuously differentiable radial profile with bounded derivative and $\phi'(0)=0$. This family includes the SE-ARD kernel (with $\phi(r)=\exp(-\tfrac{1}{2}r^2)$) and common Matérn-ARD kernels such as Matérn-$5/2$.

  Intuitively, large $\ell_j$ indicate directions along which the surrogate varies slowly, so we expect changes along those components to have a small effect on the acquisition. Under mild regularity conditions we can formalize this intuition by bounding the partial derivatives of the acquisition with respect to each coordinate.

  Fix a round $t$ and consider acquisition $\alpha_t$ and optimizer $\bm x_t^*$. Let $S_t^0 := \{ j : x_{t,j}^* \neq x_j^{\mathrm{def}}\}$ be the set of components where $\bm x_t^*$ differs from the default. For any set of components $S\subseteq S_t^0$, define
  \[
  R_S(\bm x_t^*) := P_{S_t^0 \setminus S}(\bm x_t^*),
  \]
  that is, $R_S(\bm x_t^*)$ is the configuration obtained from $\bm x_t^*$ by \emph{resetting} the components in $S$ to their default values and keeping all the others as in $\bm x_t^*$. Equivalently, $R_S(\bm x_t^*)$ corresponds to precisely pruning the components of $S$.

  Let $\Delta_t(\bm x) = \alpha_t(\bm x_t^*) - \alpha_t(\bm x)$ denote the acquisition gap at round $t$.

  \begin{assumption}[Per-component acquisition gaps, toy model]
  \label{assump:per_coord_gap}
  For each $j\in S_t^0$ there exists $\Delta_{t,j} \ge 0$ such that the acquisition gap when only component $j$ is reset satisfies
  \[
  \Delta_t\bigl(R_{\{j\}}(\bm x_t^*)\bigr) \;\ge\; \Delta_{t,j}.
  \]
  Moreover, we assume a simple superadditivity property of the gap function (i.e., the joint gap is at least the sum of individual gaps; equivalently, no positive interaction effects between coordinates):
  \[
  \Delta_t\bigl(R_S(\bm x_t^*)\bigr)
  = \alpha_t(\bm x_t^*) - \alpha_t\bigl(R_S(\bm x_t^*)\bigr)
  \;\ge\; \sum_{j\in S} \Delta_{t,j},
  \quad\forall S\subseteq S_t^0.
  \]
  \end{assumption}

  Assumption~\ref{assump:per_coord_gap} says that each component has an associated “individual gap’’ $\Delta_{t,j}$, and that the combined effect of resetting a set of components is at least the sum of their individual gaps. This rules out positive interaction effects in which resetting two components together could be cheaper than resetting them separately, and is best viewed as a simplified model for intuition rather than a realistic assumption in complex BO problems.

  Let $S_t^{\mathrm{prune}}\subseteq S_t^0$ denote the set of components that BONSAI resets to default at round $t$, so that $\tilde{\bm x}_t = R_{S_t^{\mathrm{prune}}}(\bm x_t^*)$.

  \begin{proposition}[Thresholds constrain how many high-impact components can be pruned]
  \label{prop:tau_sparsity}
  Suppose Assumption~\ref{assump:per_coord_gap} holds at round $t$, and that the BONSAI applies the relative gap rule so that $\Delta_t(\tilde{\bm x}_t) \le \rho_t \tilde\alpha_t(\bm x_t^*)$. Then
  \[
  \sum_{j\in S_t^{\mathrm{prune}}}\Delta_{t,j} \le \rho_t \tilde\alpha_t(\bm x_t^*).
  \]
  In particular, if $\Delta_{t,j} \ge \Delta_{\min}>0$ for all $j\in S_t^{\mathrm{prune}}$, then
  \[
  |S_t^{\mathrm{prune}}|\le \frac{\rho_t \tilde\alpha_t(\bm x_t^*)}{\Delta_{\min}}.
  \]
  \end{proposition}

  \begin{proof}
  By construction, $\tilde{\bm x}_t$ is obtained from $\bm x_t^*$ by resetting the components in $S_t^{\mathrm{prune}}$ to their default values, that is, $\tilde{\bm x}_t = R_{S_t^{\mathrm{prune}}}(\bm x_t^*)$. Applying Assumption~\ref{assump:per_coord_gap} with $S = S_t^{\mathrm{prune}}$ yields
  \[
  \Delta_t(\tilde{\bm x}_t) = \Delta_t\bigl(R_{S_t^{\mathrm{prune}}}(\bm x_t^*)\bigr)
  \ge \sum_{j\in S_t^{\mathrm{prune}}}\Delta_{t,j}.
  \]
  The relative gap rule implies $\Delta_t(\tilde{\bm x}_t)\le \rho_t \tilde{\alpha}_t(\bm x_t^*)$. Letting $\tau_t := \rho_t \tilde{\alpha}_t(\bm x_t^*)$ denote the (per-round) absolute threshold, we obtain
  \[
  \sum_{j\in S_t^{\mathrm{prune}}}\Delta_{t,j} \le \tau_t.
  \]
  If each $\Delta_{t,j}\ge \Delta_{\min}$, then
  \[
  |S_t^{\mathrm{prune}}|\,\Delta_{\min}
  \le \sum_{j\in S_t^{\mathrm{prune}}}\Delta_{t,j}
  \le \rho_t \tilde{\alpha}_t(\bm x_t^*),
  \]
  which proves the claimed bound.
  \end{proof}

  \begin{lemma}[Coordinate sensitivity for ARD kernels]
  \label{lemma:coord_lip_main}
  Let $\alpha_t$ be the GP-UCB acquisition at round $t$ with an ARD kernel of the form
  \[
      k(\bm x, \bm x') = \sigma_f^2 \,\phi\!\left(
          \sqrt{\sum_{j=1}^d \frac{(x_j - x_j')^2}{\ell_j^2}}
      \right),
  \]
  on a compact hyper-rectangle $\mathbb X$. Suppose $\phi$ is continuously differentiable, $\phi'$ is bounded, and $\phi'(0)=0$. Also suppose the posterior standard deviation is continuously differentiable on $\mathbb X$ and bounded below by a round-dependent constant $\underline\sigma_t>0$. Then there exists a finite constant $C_t$ (depending on the data, fitted kernel hyperparameters, $\underline\sigma_t$, and $\beta_t$, but not on the coordinate index $j$) such that for all $\bm x\in\mathbb X$ and all $j$,
  \[
      \biggl|\frac{\partial \alpha_t(\bm x)}{\partial x_j}\biggr|
      \;\le\; \frac{C_t}{\ell_j}.
  \]
  The bound is conditional on the fitted model; because $C_t$ may itself depend on the fitted hyperparameters, it is not a uniform asymptotic statement as a lengthscale varies.
  \end{lemma}

  \begin{proof}
  We sketch the argument. Differentiating the kernel with respect to the radial distance $r$ gives
  \[
  \frac{\partial k(\bm x,\bm x')}{\partial x_j}
  = \sigma_f^2\,\phi'(r)\,\frac{\partial r}{\partial x_j}, \qquad
  r = \sqrt{\sum_{m=1}^d \frac{(x_m - x_m')^2}{\ell_m^2}}.
  \]
  By the chain rule, for $r > 0$,
  \[
  \biggl|\frac{\partial r}{\partial x_j}\biggr|
  = \frac{1}{r}\,\biggl|\frac{(x_j - x_j')}{\ell_j^2}\biggr|
  \le \frac{1}{r}\cdot \frac{|x_j - x_j'|/\ell_j}{\ell_j} \le \frac{1}{\ell_j},
  \]
  using $|x_j-x_j'|/\ell_j \le r$. Hence
  \(
      \bigl|\partial k(\bm x,\bm x')/\partial x_j\bigr|
      = \sigma_f^2\,|\phi'(r)|\cdot|\partial r/\partial x_j|
      \le \sigma_f^2\,|\phi'(r)|/\ell_j.
  \)
  By the lemma's hypothesis, $\phi'$ is bounded, so the partial derivative of the kernel is uniformly bounded by $C'_k/\ell_j$ for some finite constant $C'_k$ independent of $j$. The condition $\phi'(0)=0$ makes the radial kernel differentiable at $r=0$ and extends this bound continuously there.

  The GP posterior mean and variance are linear and quadratic forms in the kernel evaluations, respectively, with coefficients determined by the training data and noise level. Differentiating these expressions and applying the kernel-derivative bound gives coordinate derivatives of order $C'_t/\ell_j$. Since $\sigma_{t-1}=\sqrt{v_{t-1}}$ and $\sigma_{t-1}\ge\underline\sigma_t>0$, dividing the variance derivative by $2\sigma_{t-1}$ preserves a finite bound of this form for $\partial\sigma_{t-1}/\partial x_j$. Since
  $
  \alpha_t(\bm x) = \mu_{t-1}(\bm x) + \beta_t\,\sigma_{t-1}(\bm x),
  $
  matching the GP-UCB definition in Section~\ref{sec:background}, we obtain the claimed bound with $C_t$ proportional to $C'_t(1+\beta_t)$.
  \end{proof}

  \begin{remark}[Mixed continuous--discrete domains]
  Lemma~\ref{lemma:coord_lip_main} assumes a continuous domain $\mathbb X \subset \mathbb R^d$ and uses the gradient $\partial \alpha_t/\partial x_j$ to bound coordinate sensitivity. Several of our experimental benchmarks (Joint NAS/HPO, Cell Network) include discrete coordinates, where this gradient bound does not apply directly. However, the BONSAI algorithm itself is gradient-free: it only evaluates the acquisition gap of single-coordinate resets and prunes whenever the gap is below the threshold (\cref{alg:bonsai}). The sparsity-recovery guarantee under perfect ARD lengthscale separation (\cref{thm:recovery_ard}) extends to discrete coordinates whenever the GP marginalizes irrelevant discrete coordinates exactly (i.e., the predictive distribution does not depend on those coordinates), in which case any reset of an irrelevant discrete coordinate produces zero acquisition gap. Lemma~\ref{lemma:coord_lip_main} should be read as motivating---in the continuous case---why ARD lengthscales control which coordinates BONSAI tends to prune; the algorithm and its empirical behavior do not depend on this gradient bound.
  \end{remark}

  % \subsection{Thresholds, Lengthscales, and Sparsity}
  % The regret bounds above characterize the optimization cost of allowing acquisition gaps, but they do not by themselves explain how the thresholds influence the \emph{sparsity} of BONSAI's recommendations. We now give a brief, concrete picture of this interaction for a broad class of ARD kernels that includes both SE-ARD and Matérn-ARD (e.g., Matérn-$5/2$). This subsection helps explain which coordinates BONSAI tends to prune under common ARD kernels.

  % We assume an ARD kernel of the form
  % \[
  %     k(\bm x, \bm x') = \sigma_f^2 \,\phi\!\left(
  %         r(\bm x, \bm x')
  %     \right),
  %     \enspace
  %     r(\bm x, \bm x') = \sqrt{\sum_{j=1}^d \frac{(x_j - x_j')^2}{\ell_j^2}},
  % \]
  % where $\ell_j > 0$ are per-component lengthscales and $\phi : [0,\infty) \to \mathbb R$ is a smooth radial profile such that $r \mapsto r\,\phi'(r)$ is bounded on $[0,\infty)$. This family includes the SE-ARD kernel (with $\phi(r) = \exp(-\tfrac{1}{2}r^2)$) and common Matérn-ARD kernels such as Matérn-$5/2$.

  % Intuitively, large $\ell_j$ indicate directions along which the surrogate varies slowly, so we expect changes along those components to have a small effect on the acquisition. Under mild regularity conditions we can formalize this intuition by bounding the partial derivatives of the acquisition with respect to each coordinate; a precise statement and proof are given in Lemma~\ref{lemma:coord_lip_main}. Let
  % \[
  %     R_{t,j} := \bigl|x^*_{t,j} - x^{\mathrm{def}}_j\bigr|
  % \]
  % denote the deviation of component $j$ from the default at round $t$.
  Define the per-coordinate deviation $R_{t,j} := |x_{t,j}^* - x_j^{\mathrm{def}}|$ at round $t$ (so $R_{t,j}=0$ for $j\notin A(\bm x_t^*)$). Lemma~\ref{lemma:coord_lip_main} then implies that changing component $j$ by $R_{t,j}$ can change the acquisition by at most
  $w_{t,j} := \frac{C_t R_{t,j}}{\ell_j},$
  for some finite constant $C_t$ that depends on the data, kernel hyperparameters, and $\beta_t$ but not on $j$.

  Now let $S \subseteq \{1,\dots,d\}$ be a set of components that we reset to their default values, and consider the candidate
  $\bm x' = R_{S}(\bm x_t^*),$
  obtained by starting from $\bm x_t^*$ and reverting the components in $S$ back to $\bm x^{\mathrm{def}}$. Resetting those components one at a time and summing the resulting changes in $\alpha_t$ yields the following bound on the acquisition gap:
  $$\Delta_t(\bm x') \;=\; \alpha_t(\bm x_t^*) - \alpha_t(\bm x')
      \;\le\; \sum_{j \in S} w_{t,j}
      \;=\; \sum_{j \in S} \frac{C_t R_{t,j}}{\ell_j}.$$

  Under the relative gap rule with threshold $\rho_t$, any subset $S$ satisfying
  $
      \sum_{j \in S} \frac{C_t R_{t,j}}{\ell_j} \le \rho_t\,\tilde\alpha_t(\bm x_t^{*})$
  defines a feasible reset projection under the relative constraint $\Delta_t(\tilde{\bm x})\le\rho_t\tilde\alpha_t(\bm x_t^{*})$. Thus $\rho_t\tilde\alpha_t(\bm x_t^{*})$ acts as a per-round budget: components with large fitted lengthscales and small deviations have a small certified cost. This statement establishes feasibility of the projection $R_S(\bm x_t^*)$; by itself it does not guarantee that sequential greedy pruning follows a path that selects exactly $S$, nor does it provide a uniform cumulative regret rate without controlling the normalized bounds across rounds.

  \subsection{Sparsity Recovery Guarantees}
  \label{appdx:sparsity_recovery}

  While the regret bounds in \cref{sec:theory} quantify the optimization cost of allowing acquisition gaps, a separate question is when BONSAI removes every irrelevant deviation from its base acquisition candidate without removing a relevant one. We establish conditions for this behavior under two structural scenarios: (1) perfect ARD separation, and (2) additive acquisition functions with separated gaps.

  Let $A_{\mathrm{true}} \subset \{1, \dots, d\}$ denote the true active set of relevant parameters, such that the objective function $f(\bm x)$ depends only on $\bm x_{A_{\mathrm{true}}}$. Let $I_{\mathrm{true}} = \{1, \dots, d\} \setminus A_{\mathrm{true}}$ denote the strictly irrelevant parameters. The ideal pruning algorithm should ensure that for the returned configuration $\tilde{\bm x}_t$, the active set satisfies $A(\tilde{\bm x}_t) \subseteq A_{\mathrm{true}}$, meaning all irrelevant parameters modified by the inner optimizer are successfully reverted to their defaults.

  \subsubsection{Scenario 1: Recovery via Perfect ARD Separation}

  Our first result assumes perfect separation between relevant and irrelevant coordinates in the fitted ARD kernel. It does not require additivity of the acquisition function.

  \begin{assumption}[Perfect Lengthscale Separation]
  \label{assum:perfect_ard}
  The GP hyperparameter optimization successfully identifies the strictly irrelevant dimensions, such that for all $j \in I_{\mathrm{true}}$, the learned ARD lengthscale $\ell_j = \infty$. For all relevant dimensions $j \in A_{\mathrm{true}}$, the lengthscales are finite, $\ell_j < \infty$.
  \end{assumption}

  \begin{assumption}[Relevant Dimension Gap]
  \label{assum:relevant_gap}
  Let $S_{\mathrm{irrel}}=I_{\mathrm{true}}\cap A(\bm x_t^*)$ be the irrelevant parameters modified by the base acquisition candidate. The threshold $\tau_t=\rho_t\tilde\alpha_t(\bm x_t^*)$ is strictly smaller than the acquisition gap of reverting any single modified relevant parameter alongside the irrelevant ones:
  $\tau_t<\min_{j\in A_{\mathrm{true}}\cap A(\bm x_t^*)}\Delta_t(R_{S_{\mathrm{irrel}}\cup\{j\}}(\bm x_t^*))$.
  If $A_{\mathrm{true}}\cap A(\bm x_t^*)$ is empty, the minimum is defined as $+\infty$ and the condition is vacuous.
  \end{assumption}

  \begin{theorem}[Exact Recovery under Perfect ARD]
  \label{thm:recovery_ard}
  Under Assumptions~\ref{assum:perfect_ard} and~\ref{assum:relevant_gap}, applying BONSAI to the acquisition maximizer $\bm x_t^*$ guarantees exact sparsity recovery with respect to the queried candidate: all irrelevant components are pruned, and no relevant components are pruned. That is, $A(\tilde{\bm x}_t) = A(\bm x_t^*) \cap A_{\mathrm{true}}$.
  \end{theorem}

  \begin{proof}
  Under Assumption~\ref{assum:perfect_ard}, an irrelevant coordinate with $\ell_j=\infty$ drops out of the ARD distance. Kernel evaluations, the GP posterior, and therefore GP-UCB are invariant to that coordinate. Resetting any subset of modified irrelevant coordinates consequently incurs zero acquisition loss: $\Delta_t(R_{S_{\mathrm{irrel}}}(\bm x_t^*))=0$.

  Because $\tau_t\ge0$, these irrelevant-only resets are feasible. Each remains zero-cost along the sequential path, so greedy BONSAI cannot terminate while a modified irrelevant coordinate remains. This establishes support containment, but not global minimum cardinality: interactions can make a joint reset of \emph{relevant}coordinates feasible even when no single relevant reset is feasible.

  By \cref{assum:relevant_gap}, pruning any additional relevant component $j \in A_{\mathrm{true}}$ results in an acquisition gap strictly greater than $\tau_t$, violating the threshold constraint. Therefore, BONSAI exactly recovers the true relevant components modified by the inner optimizer, yielding $A(\tilde{\bm x}_t) = A(\bm x_t^*) \setminus S_{\mathrm{irrel}} = A(\bm x_t^*) \cap A_{\mathrm{true}}$.
  \end{proof}

  \subsubsection{Scenario 2: Recovery under Additive Acquisition Models}

  While perfect lengthscale learning represents an asymptotic ideal, we can also prove sparsity recovery in finite-sample regimes if the acquisition function decomposes additively, formalizing the intuition from the additive model in \cref{appdx:sparsity}.

  \begin{assumption}[Additive Acquisition]
  \label{assum:additive_acq}
  The acquisition function decomposes additively over coordinates: $\alpha_t(\bm x) = \sum_{j=1}^d \alpha_{t,j}(x_j)$. Consequently, the total acquisition gap for pruning a set of components $S \subseteq A(\bm x_t^*)$ (i.e., resetting those components to their defaults) is exactly $\Delta_t(R_S(\bm x_t^*)) = \sum_{j \in S} \Delta_{t,j}$, where $\Delta_{t,j} = \alpha_{t,j}(x_{t,j}^*) - \alpha_{t,j}(x_j^{\mathrm{def}})$.
  \end{assumption}

  \textbf{Remark:} While standard acquisition functions like Expected Improvement (EI) are not strictly additive due to non-linear transformations, \cref{assum:additive_acq} holds exactly for Additive UCB formulations \citep[e.g.,][]{pmlr-v37-kandasamy15} designed for high-dimensional spaces, as well as for posterior samples utilized in Thompson Sampling under Additive GPs. For standard acquisition functions, this assumption serves as a localized, first-order Taylor approximation of the acquisition landscape near the optimum, consistent with the simplified model in \cref{appdx:sparsity}.

  \begin{assumption}[Gap Separation]
  \label{assum:gap_separation}
  There is a strict separation in the per-component acquisition gaps between relevant and irrelevant parameters. There exist bounds $\epsilon \ge 0$ and $\Delta_{\min} > 0$ such that:
  \begin{itemize}
      \item For all $j \in I_{\mathrm{true}} \cap A(\bm x_t^*)$, $\Delta_{t,j} \le \epsilon$ (irrelevant variables have minimal acquisition gaps).
      \item For all $j \in A_{\mathrm{true}} \cap A(\bm x_t^*)$, $\Delta_{t,j} \ge \Delta_{\min}$ (relevant variables have significant acquisition gaps).
  \end{itemize}
  \end{assumption}

  \begin{theorem}[Sparsity Recovery under Gap Separation]
  \label{thm:recovery_additive}
  Under Assumptions~\ref{assum:additive_acq} and~\ref{assum:gap_separation}, if the BONSAI threshold parameter $\rho_t$ is chosen such that the absolute tolerance $\tau_t=\rho_t\tilde{\alpha}_t(\bm x_t^*)$ satisfies
  $|I_{\mathrm{true}}|\epsilon \le \tau_t < \Delta_{\min},$
  then sequential greedy BONSAI (\cref{alg:bonsai}) recovers the support of the base candidate relative to the true active set: $A(\tilde{\bm x}_t)=A(\bm x_t^*)\cap A_{\mathrm{true}}$.
  \end{theorem}

  \begin{proof}
  By \cref{assum:additive_acq}, the total acquisition gap for pruning all modified irrelevant variables simultaneously is $\sum_{j \in I_{\mathrm{true}} \cap A(\bm x_t^*)}\Delta_{t,j}\le|I_{\mathrm{true}}|\epsilon$. Because $\tau_t\ge|I_{\mathrm{true}}|\epsilon$, the state where all irrelevant variables are reset is feasible.

  The sequential greedy algorithm operates by iteratively resetting the component that yields the smallest feasible gap. Because individual irrelevant components have gaps bounded by $\epsilon \le |I_{\mathrm{true}}|\epsilon \le \tau_t$, the algorithm will iteratively successfully prune them.

  Conversely, to prune even a single relevant variable $k \in A_{\mathrm{true}}$, the candidate must incur an additional gap of at least $\Delta_{\min}$. Because the threshold dictates $\tau_t < \Delta_{\min}$, no relevant variable can ever be pruned without violating the threshold constraint, regardless of the sequence of prior resets. Thus, the algorithm halts exactly when all $j \in I_{\mathrm{true}} \cap A(\bm x_t^*)$ have been pruned, isolating the true sparse structure.
  \end{proof}

  \section{BONSAI Algorithm}
  \FloatBarrier
  \label{appdx:bonsai_alg}

  \paragraph{Incremental acquisition functions.}
  To make the relative threshold $\rho_t$ well-posed and comparable across acquisitions, we apply it to an \emph{incremental} acquisition value above a baseline. We subtract the maximum \emph{current} acquisition value across previously evaluated designs from the acquisition function: $\tilde{\alpha}_t(\bm x) = \alpha_t(\bm x) - b_t$, where $b_t = \max_{\bm x \in \mathcal{D}_{t-1}}\alpha_t(\bm x)$ and $\mathcal{D}_{t-1}$ denotes the set of designs evaluated before round $t$, before applying the relative rule.\footnote{\label{fn:bonsai_init}We assume there are always previously evaluated designs (e.g., a space-filling initialization), so $\mathcal{D}_{t-1}\neq\emptyset$ for all $t\ge 1$.} We require $\alpha_t(\bm x_t^*)\ge b_t$, which exact acquisition maximization guarantees; an approximate solver can enforce this by comparing its output with the baseline-attaining previous design. Both $\Delta_t$ and $\tilde\alpha_t$ are invariant to adding a constant to $\alpha_t$, so no nonnegativity convention is needed. For log-acquisition functions \citep{logei}, we apply BONSAI to exponentiated acquisition values (undoing the log transform) before forming $b_t$ and applying the relative rule: $\Delta(P_S(\bm x^{*})) \le \rho \,\tilde\alpha(\bm x^{*})$ with $0 \le \rho < 1$. Exact pruning chooses $S^\star \in \arg\min_{S \subseteq A(\bm x^*)}\bigl\{ |S| : \Delta(P_S(\bm x^*)) \le \rho \,\tilde\alpha(\bm x^*) \bigr\}$ and returns $P_{S^\star}(\bm x^*)$; Algorithm~\ref{alg:bonsai} is the scalable greedy approximation used in our main experiments.
  \begin{algorithm}[h]
    \caption{BONSAI (sequential pruning with gaps)}
    \label{alg:bonsai}
  \begin{algorithmic}[1]
    \STATE {\bfseries Inputs:} Acquisition functions $(\alpha_t)_{t=1}^T$, default $\bm x^{\mathrm{def}}$, relative thresholds $(\rho_t)_{t=1}^T$
    \FOR{$t=1$ {\bfseries to} $T$}
          \STATE $\bm x_t^* \approx \argmax_{\bm x \in \mathbb X} \alpha_t(\bm x)$
          \STATE $\tilde{\bm x}_t \leftarrow \bm x_t^*$
          \STATE $\mathcal D \leftarrow \{ j : x_{t,j}^* \neq x_j^{\mathrm{def}} \}$ % components differing from default
          \STATE $b_t \leftarrow \max_{\bm x \in \mathcal{D}_{t-1}}\alpha_t(\bm x)$ \;\; ($\mathcal{D}_{t-1}$ is the set of previously evaluated designs, which is non-empty by Footnote~\ref{fn:bonsai_init})
          \STATE Define $\tilde\alpha_t(\bm x) \leftarrow \alpha_t(\bm x) - b_t$
          \STATE Define $\Delta_t(\bm x) \leftarrow \alpha_t(\bm x_t^{*}) - \alpha_t(\bm x)$
          \WHILE {$\mathcal D \neq \emptyset$}
              \STATE $\text{gap}_\text{best} \leftarrow +\infty$, $j_\text{best} \leftarrow -1$
              \FOR{$j \in \mathcal D$}
    \STATE Let $\bm x' \leftarrow \tilde{\bm x}_t$ with component $j$ reset to $x_j^{\mathrm{def}}$
    \STATE $g_j \leftarrow \Delta_t(\bm x')$
    \STATE $\mathsf{feasible} \leftarrow \bigl(g_j \le \rho_t\,\tilde\alpha_t(\bm x_t^{*})\bigr)$
    \IF{$\mathsf{feasible}$ {\bfseries and} $g_j < \text{gap}_\text{best}$}
                      \STATE $\text{gap}_\text{best} \leftarrow g_j$
                      \STATE $j_\text{best} \leftarrow j$
                  \ENDIF
              \ENDFOR
              \IF{$j_\text{best} \neq -1$}
                  \STATE Reset component $j_\text{best}$ in $\tilde{\bm x}_t$ to $x_{j_\text{best}}^{\mathrm{def}}$
                  \STATE $\mathcal D \leftarrow \mathcal D \setminus \{j_\text{best}\}$
              \ELSE
                  \STATE \textbf{break}
              \ENDIF
          \ENDWHILE
          \STATE Evaluate $y_t = f(\tilde{\bm x}_t) + \varepsilon_t$
          \STATE Update the surrogate with $(\tilde{\bm x}_t, y_t)$
      \ENDFOR
  \end{algorithmic}
  \end{algorithm}
  \FloatBarrier
  \section{Additional Empirical Evaluation and Experiment details}
  \label{appdx:empirical}
  \FloatBarrier
  \subsection{Minimal-Intervention Marker for Multi-Objective Problems}
  \label{appdx:moo_marker}
  For multi-objective problems, the per-method marker in Figures~\ref{fig:synthetic}--\ref{fig:real} is an empirical finite-run analogue of the population target in Section~\ref{sec:problem_setting}. For each replication and sparsity level $k$, let $\widehat H_k$ be the hypervolume of the frontier formed by explored points with at most $k$ active dimensions. For each method configuration we average its per-replication best hypervolume, define $\widehat H_{\max}$ as the largest such mean across the tested methods, and let $\widehat H_{\mathrm{def}}$ be the mean default hypervolume. We use the threshold $\widehat H_{\max}-0.2(\widehat H_{\max}-\widehat H_{\mathrm{def}})$. Each replication contributes its smallest $k$ that clears this threshold, or its best sparse-frontier point as a fallback if it never clears the threshold. We snap these selected points to the closest point on the method's mean sparse frontier and plot their median with 25th--75th percentile error bars in both dimensions. Replications represented by fallback points remain classified as infeasible. The single-objective markers use the same protocol with $f_{\max}$ and $f_{\mathrm{def}}$. Because $\widehat H_{\max}$ is the best achieved mean rather than the unknown $H_d^\star$, these markers do not certify recovery of $k_\epsilon^\star$.
  \FloatBarrier
  \subsection{Sequential Optimization Generation Times}
  \label{appdx:gen_times}
  As shown in Table~\ref{table:gen_times_bonsai_map_saas_problems_combined}, BONSAI has sequential generation times comparable to Standard BO and is faster on average than PiBO and the other default-aware methods. PiBO remains far less expensive than the homotopy-based IR/ER/SEBO methods on most problems.
   \begin{table*}[t]
  \centering
  \begin{small}
  \begin{sc}
  \resizebox{\textwidth}{!}{%
  \begin{tabular}{lccccc}
  \toprule
  & IR-$L_0$ & ER-$L_0$ & SEBO & BONSAI EI & PiBO \\ \midrule
  Joint NAS/HPO (26d) & 25.0x ($\pm$ 1.7x) & 27.3x ($\pm$ 2.1x) & 52.6x ($\pm$ 6.1x) & $\bm{0.9\text{x}}$ ($\pm$ 0.0x) & 2.3x ($\pm$ 0.2x) \\
  Optical Design (146d) & -- & 3.4x ($\pm$ 0.2x) & 2.6x ($\pm$ 0.2x) & $\bm{2.0\text{x}}$ ($\pm$ 0.2x) & 2.1x ($\pm$ 0.1x) \\
  LassoDNA (180d) & 3.8x ($\pm$ 0.4x) & 5.3x ($\pm$ 0.5x) & 3.4x ($\pm$ 0.5x) & 3.1x ($\pm$ 0.3x) & $\bm{1.6\text{x}}$ ($\pm$ 0.2x) \\
  Cell Network (30d) & -- & 25.9x ($\pm$ 2.3x) & 99.1x ($\pm$ 13.3x) & $\bm{1.1\text{x}}$ ($\pm$ 0.1x) & 1.6x ($\pm$ 0.1x) \\
  Branin 50d (2d) & 2.9x ($\pm$ 0.5x) & 3.0x ($\pm$ 0.3x) & 4.0x ($\pm$ 0.5x) & $\bm{1.0\text{x}}$ ($\pm$ 0.1x) & 2.9x ($\pm$ 0.2x) \\
  Hartmann 50d (6d) & 3.0x ($\pm$ 0.2x) & 5.6x ($\pm$ 0.6x) & 8.4x ($\pm$ 0.8x) & $\bm{1.5\text{x}}$ ($\pm$ 0.1x) & 2.7x ($\pm$ 0.3x) \\
  PressureVessel 50d (4d) & 2.3x ($\pm$ 0.3x) & 3.0x ($\pm$ 0.2x) & 2.2x ($\pm$ 0.3x) & 1.5x ($\pm$ 0.1x) & $\bm{1.2\text{x}}$ ($\pm$ 0.1x) \\
  DTLZ2 50d (6d) & -- & 4.0x ($\pm$ 0.3x) & 4.0x ($\pm$ 0.4x) & $\bm{1.4\text{x}}$ ($\pm$ 0.1x) & 3.4x ($\pm$ 0.2x) \\
  \bottomrule
  \end{tabular}%
  }
  \end{sc}
  \end{small}
  \caption{Mean paired-seed generation-time ratio to Vanilla BO ($\pm$ 2 standard errors). The fastest default-aware method is shown in bold.}\label{table:gen_times_bonsai_map_saas_problems_combined}
  \end{table*}

  \paragraph{Per-problem wall-time notes.}
  SEBO is omitted from Optical Design in the batch ($q=5$) setting (Table~\ref{table:gen_times_bonsai_map_saas_problems_combined_q5}) since a single replication did not finish within a 4-day wall-time budget, due to the cost of homotopy continuation. On Joint NAS/HPO (a mixed integer/continuous space), BONSAI's reported generation time is slightly below Vanilla BO's because pruning yields sparser data, which in practice accelerates downstream GP fits and acquisition optimization on the same problem (the post-acquisition pruning loop adds only a small constant overhead per outer iteration).

  \paragraph{Active-dimension tolerance for continuous-relaxation baselines.}
  For every method, the benchmark sparse-frontier computation counts a coordinate as inactive when $|x_j-x_j^{\mathrm{def}}|\leq 10^{-2}$ in the experiment's parameter coordinates. This common numerical tolerance accommodates values returned by continuous acquisition optimizers that are very close, but not bitwise equal, to the default.
  \subsection{Method Details}
  For mixed search spaces (Cell Network and Joint NAS/HPO), we use Ax's default dispatch to optimize the acquisition function using the mixed-alternating optimizer, which interleaves local search on the discrete parameters with L-BFGS-B steps on the continuous parameters \citep{olson2025ax} for all methods. The mixed-alternating optimizer makes IR and ER very slow as shown in the wall times on problems with discrete parameters (Cell Network and Joint NAS/HPO).

  Priors on the GP hyperparameters are given in Appendix~\ref{appdx:map_saas}.

  For PiBO, let $\bm t\in[0,1]^d$ be the normalized representation of the same pruning target $\bm x^{\mathrm{def}}$ used by BONSAI. We use the product prior
  $\pi(\bm x)=\prod_{i=1}^d \mathrm{Laplace}(x_i;t_i,b)\propto\exp(-\|\bm x-\bm t\|_1/b)$
  with shared scale $b=0.1$. Following \citet{hvarfner2022pibo}, at the $n$th post-initialization BO iteration we multiply the acquisition value by $\pi(\bm x)^{\beta/n}$, where $\beta=N/10$ and $N$ is the total number of post-initialization BO iterations. PiBO uses the same ensemble MAP-SAAS surrogate, 20-trial initialization (one default plus 19 Sobol samples), evaluation budget, and 20 replications as the other sequential methods.

  For IR and ER, we use 30 steps in homotopy continuation, with $a_\text{start}=0.2$ and $a_\text{end}=10^{-3}$, which follows the recommendation in \citet{sebo}.

  For constrained problems, if there is no feasible previously evaluated point, we optimize the probability of feasibility as the acquisition function (as in \citet{olson2025ax}) and do not perform pruning or leverage IR/ER.
  \subsection{Problem Details}

  \textbf{Synthetic Benchmarks.} Each synthetic problem is embedded in a higher-dimensional search space by adding extra irrelevant parameters: Branin (2d), Hartmann (6d), PressureVessel (4d, 4 constraints)~\citep{pressure_vessel}, and DTLZ2 (6d, 2-objective)~\citep{dtlz}, all embedded in 50d spaces. Additional results on lower-dimensional benchmarks are provided below; these include \emph{Branin 15d} (Branin's 2d objective embedded in 15 dimensions), \emph{Hartmann 15d} (Hartmann's 6d objective embedded in 15 dimensions), and \emph{BraninCurrin} (the standard 2d, 2-objective benchmark combining the Branin and Currin functions; not embedded), used in Appendix~\ref{appdx:bonsai_exact} where exact pruning is computationally feasible.

  \textbf{Joint NAS/HPO} is a real-world joint neural architecture search and hyperparameter optimization problem originating from a multi-task, multi-label neural network model at a large web services firm. The problem has 26 parameters that include the widths and depths of different blocks (integer parameters) and initial weights for layer norms (continuous parameters). Training and evaluating the model is computationally intensive, so sample efficiency is paramount. For this benchmark, we leverage a GP surrogate fitted to data collected from a real optimization run.

  \textbf{Optical Design} is a real-world problem where the goal is to optimize 146 continuous parameters that control the surface morphology and geometry of optical components of a see-through display for augmented reality \citep{morbo}. The goal is to optimize two objectives: display efficiency and quality. Evaluating a design requires computationally intensive physical simulations; however, for the purposes of this benchmark, we leverage a neural network surrogate fit to real data collected from optimization runs.

  \textbf{LassoDNA}~\citep{2022lassobench} involves tuning 180 parameters (each in $[-1,1]$) corresponding to the penalty for each feature in a weighted Lasso regression to minimize MSE on the DNA dataset~\citep{libsvm}. The all-zero penalty vector is the default. LassoDNA is available at \url{https://github.com/ksehic/LassoBench}, and the DNA dataset is available under a BSD-3-Clause license via LIBSVM at \url{https://github.com/cjlin1/libsvm}.

  \textbf{Cell Network} is a 30d optimization problem~\citep{dreifuerst2021} where the goal is to optimize the transmission power (an integer in $\{0,\dots,10\}$) and downtilt (a continuous parameter in $[30, 50]$) for each of 15 antennas in a cellular network to minimize two objectives: under-coverage and over-coverage. The benchmark implementation is available at \url{https://github.com/Ryandry1st/CCO-in-ORAN}.

  \paragraph{Default construction and availability.}
  Each PiBO center is the normalized form of the exact input-space pruning target used by BONSAI; we do not select or tune a separate center for PiBO. Synthetic embedded problems use the search-space center, LassoDNA uses the all-zero penalty vector, and the ranking problems use their A/B-test control configuration. The Joint NAS/HPO, Optical Design, and ranking surrogates are fit to proprietary production or simulation data and cannot be released; the public LassoDNA, DNA, and Cell Network assets are linked above.

  \subsection{Reference Points for MOO}
  For the OpticalDesign problem, we follow \citet{morbo}: we scale each objective by the corresponding component of the original reference point so that the transformed reference point is $(1,1)$.

  For CellNetwork, we use $[0.35, 0.35]$ as the reference point.

  For the real-world MOO Ranking (10d) problem, we use $[0.0, 0.0]$ as the reference point in normalized objective coordinates (objectives are negated for maximization and shifted so the default $\bm x^{\mathrm{def}}$ outcome is non-negative).

  For synthetic problems, we use default reference points for each in BoTorch \citep{balandat2020botorch}, which are commonly used in multi-objective Bayesian optimization papers \citep{qehvi, qnehvi, hvkg}.
  \subsubsection{Real-world Ranking Problem}
  \label{appdx:additional_ranking_problems}

  \textbf{MOO Ranking} is a value modeling tuning problem from a major web service, where there are 10 continuous parameters, and the goal is to learn the Pareto frontier between 2 engagement metrics. Again, we leverage a GP surrogate fit to data from real A/B tests \citep{quickbo}.

  \textbf{Results}
  The sequential results are shown in Figure~\ref{fig:ranking_extra}. BONSAI reaches higher final HV than PiBO, and both are $\epsilon$-feasible in all 20 replications. BONSAI is faster than PiBO, IR, and ER (Table~\ref{table:gen_times_bonsai_map_saas_problems_extra}).
  \begin{figure*}[h]
    \centering
    \includegraphics[width=\linewidth]{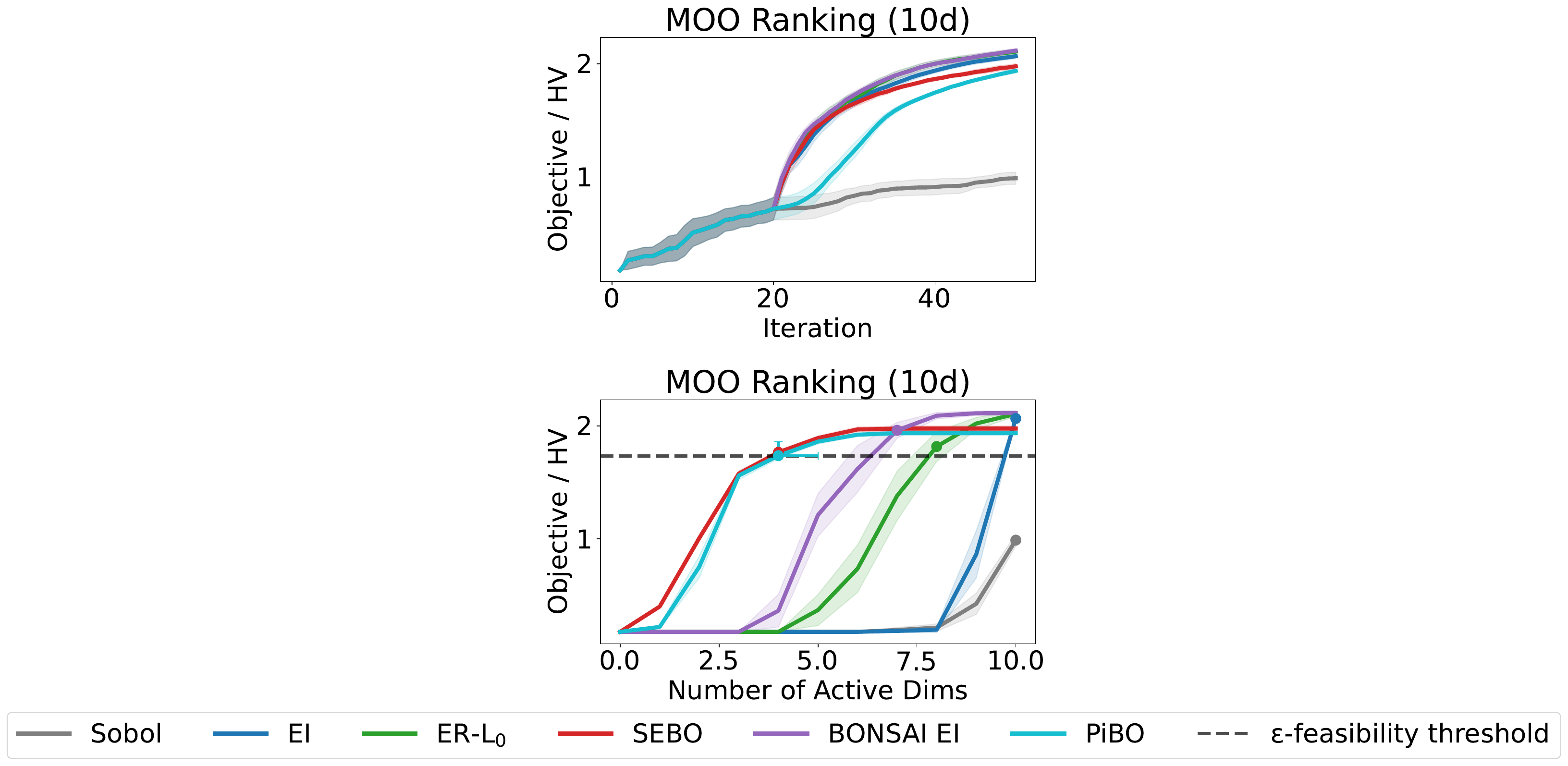}
    \caption{Real-world MOO Ranking problem. Top: objective or HV. Bottom: best objective or HV for each number of active dimensions. The single marker for each method is the snapped median summary with 25th--75th percentile error bars defined in Appendix~\ref{appdx:moo_marker}.}
    \label{fig:ranking_extra}
  \end{figure*}
   \begin{table*}[t]
  \centering
  \begin{small}
  \begin{sc}
  \resizebox{\textwidth}{!}{%
  \begin{tabular}{lccccc}
  \toprule
  & IR-$L_0$ & ER-$L_0$ & SEBO & BONSAI EI & PiBO \\ \midrule
  MOO Ranking (10d) & -- & 3.8x ($\pm$ 0.2x) & 5.7x ($\pm$ 0.5x) & $\bm{1.0\text{x}}$ ($\pm$ 0.0x) & 2.8x ($\pm$ 0.1x) \\
  \bottomrule
  \end{tabular}%
  }
  \end{sc}
  \end{small}
  \caption{Mean paired-seed generation-time ratio to Vanilla BO ($\pm$ 2 standard errors). The fastest default-aware method is shown in bold.}\label{table:gen_times_bonsai_map_saas_problems_extra}
  \end{table*}

  \FloatBarrier
  \subsection{Batch Optimization}
  \label{appdx:batch_optimization}
  We evaluate BONSAI in the batch setting where at each iteration, $q=5$ points are generated. We find that BONSAI performs well in the batch setting too.

  \begin{figure*}[h]
    \centering
    \includegraphics[width=\linewidth]{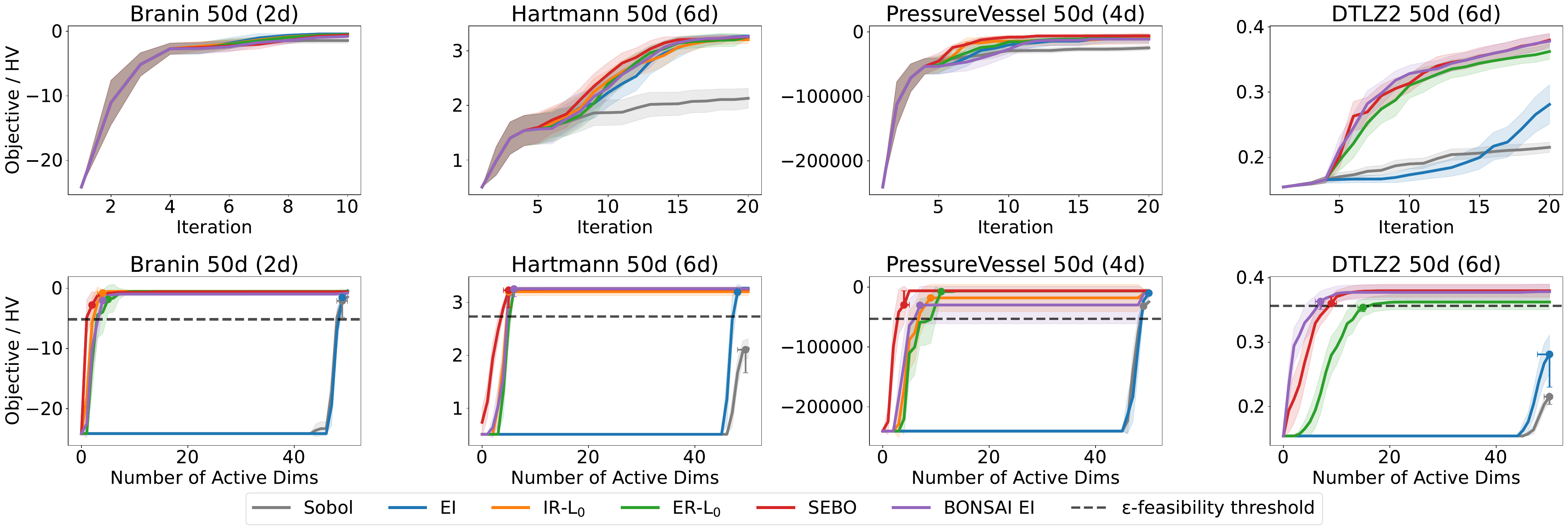}
    \caption{Optimization performance of BONSAI with batch optimization ($q=5$) on synthetic problems. Top row: Objective or HV. Bottom row: Best Objective (or HV) value for each level of active dimensions. For MOO, the HV at sparsity level $k$ is the hypervolume of the feasible Pareto frontier computed over all evaluated points with at most $k$ active dimensions.}
    \label{fig:synthetic_q5}
  \end{figure*}
  \begin{figure*}[h]
    \centering
    \includegraphics[width=\linewidth]{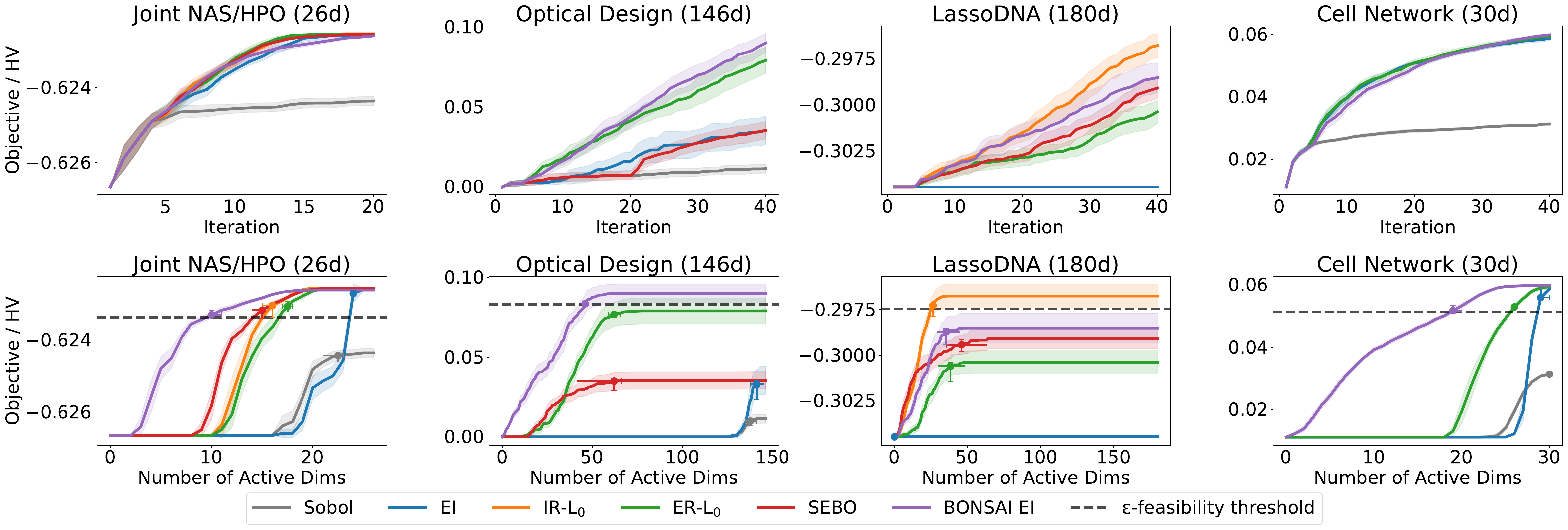}
    \caption{Optimization performance of BONSAI with batch optimization ($q=5$) on real-world problems. Top row: Objective or HV. Bottom row: Best Objective (or HV) value for each level of active dimensions. For MOO, the HV at sparsity level $k$ is the hypervolume of the feasible Pareto frontier computed over all evaluated points with at most $k$ active dimensions.}
    \label{fig:real_q5}
  \end{figure*}
   \begin{table*}[t]
  \centering
  \begin{small}
  \begin{sc}
  \resizebox{\textwidth}{!}{%
  \begin{tabular}{lcccc}
  \toprule
  & IR-$L_0$ & ER-$L_0$ & SEBO & BONSAI EI \\ \midrule
  Joint NAS/HPO (26d) & 27.4x ($\pm$ 1.0x) & 27.2x ($\pm$ 1.1x) & 78.5x ($\pm$ 6.0x) & $\bm{1.0\text{x}}$ ($\pm$ 0.0x) \\
  Optical Design (146d) & -- & 3.4x ($\pm$ 0.3x) & $\bm{2.4\text{x}}$ ($\pm$ 0.2x) & 2.8x ($\pm$ 0.2x) \\
  LassoDNA (180d) & 3.0x ($\pm$ 0.2x) & 4.5x ($\pm$ 0.3x) & 5.5x ($\pm$ 0.7x) & $\bm{2.4\text{x}}$ ($\pm$ 0.2x) \\
  Cell Network (30d) & -- & 23.7x ($\pm$ 1.7x) & -- & $\bm{1.0\text{x}}$ ($\pm$ 0.1x) \\
  Branin 50d (2d) & 3.1x ($\pm$ 0.3x) & 2.4x ($\pm$ 0.2x) & 8.4x ($\pm$ 1.4x) & $\bm{1.1\text{x}}$ ($\pm$ 0.1x) \\
  Hartmann 50d (6d) & 3.2x ($\pm$ 0.2x) & 5.0x ($\pm$ 0.5x) & 17.2x ($\pm$ 1.6x) & $\bm{1.5\text{x}}$ ($\pm$ 0.1x) \\
  PressureVessel 50d (4d) & 3.4x ($\pm$ 0.3x) & 3.5x ($\pm$ 0.2x) & 3.5x ($\pm$ 0.3x) & $\bm{1.4\text{x}}$ ($\pm$ 0.1x) \\
  DTLZ2 50d (6d) & -- & 3.9x ($\pm$ 0.2x) & 11.0x ($\pm$ 1.3x) & $\bm{1.7\text{x}}$ ($\pm$ 0.1x) \\
  \bottomrule
  \end{tabular}%
  }
  \end{sc}
  \end{small}
  \caption{Mean paired-seed generation-time ratio to Vanilla BO ($\pm$ 2 standard errors). The fastest default-aware method is shown in bold.}\label{table:gen_times_bonsai_map_saas_problems_combined_q5}
  \end{table*}

  \FloatBarrier
  \subsection{BONSAI with UCB}
  \label{appdx:bonsai_ucb}
  We compare BONSAI (with qLogNEI) to BONSAI with UCB using $\rho_t=1/t$, which has the same asymptotic order as the schedule in Corollary~\ref{cor:schedules}. Its first value equals one; the theorem applies from $t=2$, while this single initial modeled round contributes only a finite constant. We set the exploration parameter $\beta_t$ in UCB to be $\beta_t=\sqrt{\log(t + 2)}$ as in \citet{pmlr-v286-kim25b}. To apply UCB to multi-objective problems, we use random hypervolume scalarizations \citep{pmlr-v119-zhang20i}, which reduce the multi-objective acquisition to a single-objective UCB at each iteration. Figures~\ref{fig:ei_vs_ucb_synthetic} and \ref{fig:ei_vs_ucb_realworld} show that UCB performs similarly to EI on most problems, but UCB performs worse on LassoDNA. Since EI acquisition functions (qLogNEI and qLogNEHVI) perform well across the board, we opt to use EI acquisition functions for the main-text experiments.
  \begin{figure*}[h]
    \centering
    \includegraphics[width=\linewidth]{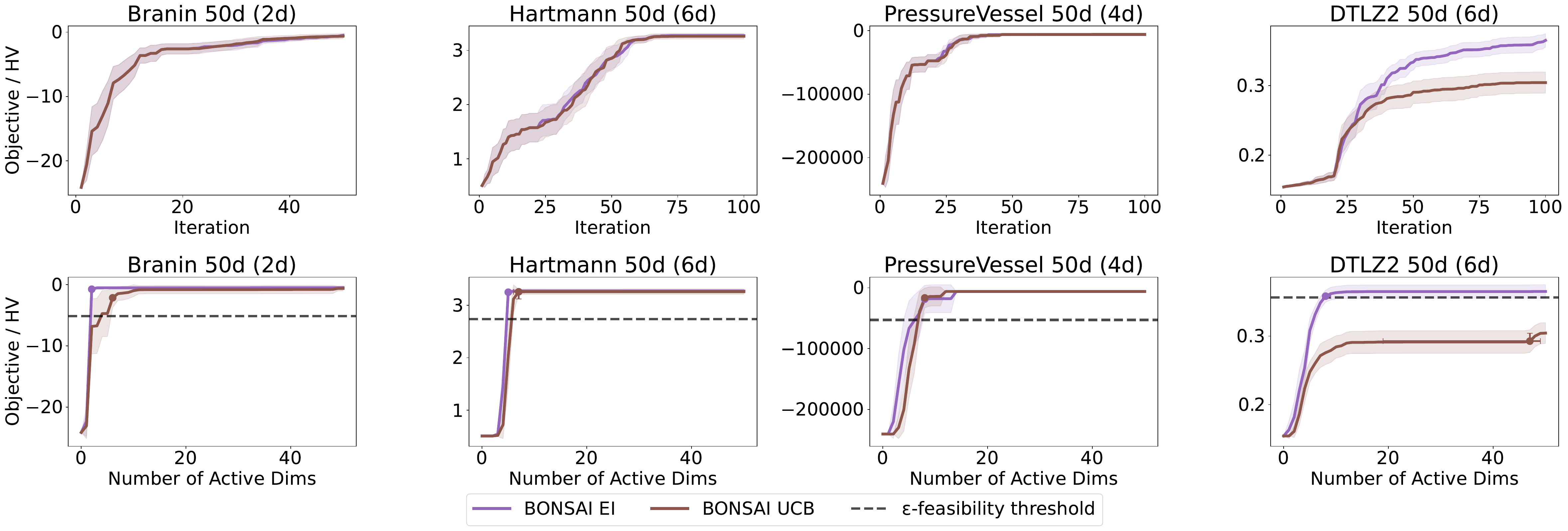}
    \caption{Comparison of BONSAI (with qLogNEI) to BONSAI with UCB on synthetic problems. Top row: Objective. Bottom row: Best Objective value for each level of active dimensions.}
    \label{fig:ei_vs_ucb_synthetic}
  \end{figure*}
  \begin{figure*}[h]
    \centering
    \includegraphics[width=\linewidth]{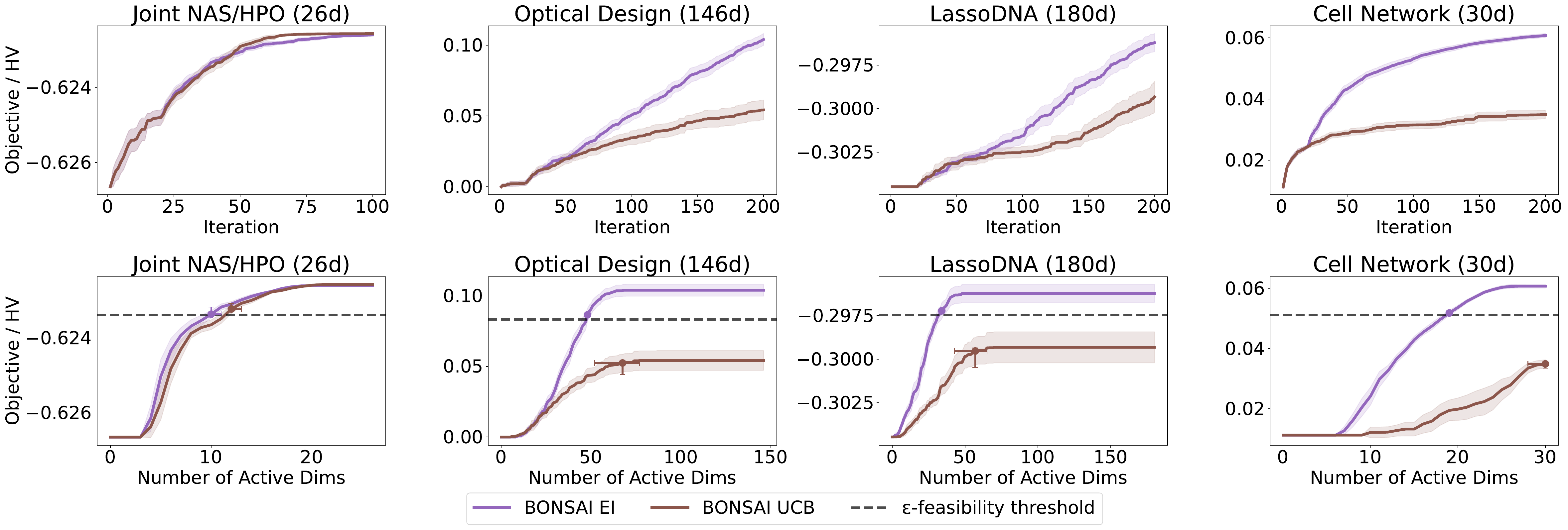}
    \caption{Comparison of BONSAI (with qLogNEI) to BONSAI with UCB on real-world problems. Top row: Objective. Bottom row: Best Objective  value for each level of active dimensions.}
    \label{fig:ei_vs_ucb_realworld}
  \end{figure*}
  \FloatBarrier
  \subsection{Sensitivity with respect to $\rho$}
  \label{appdx:bonsai_rho_sens}
  We find that BONSAI is fairly robust with respect to the choice of $\rho$, as shown in Figures~\ref{fig:rho_sens_synthetic} and \ref{fig:rho_sens_real}.  We use $\rho=0.2$ in the experiments in the main text.
  \begin{figure*}[h]
    \centering
    \includegraphics[width=\linewidth]{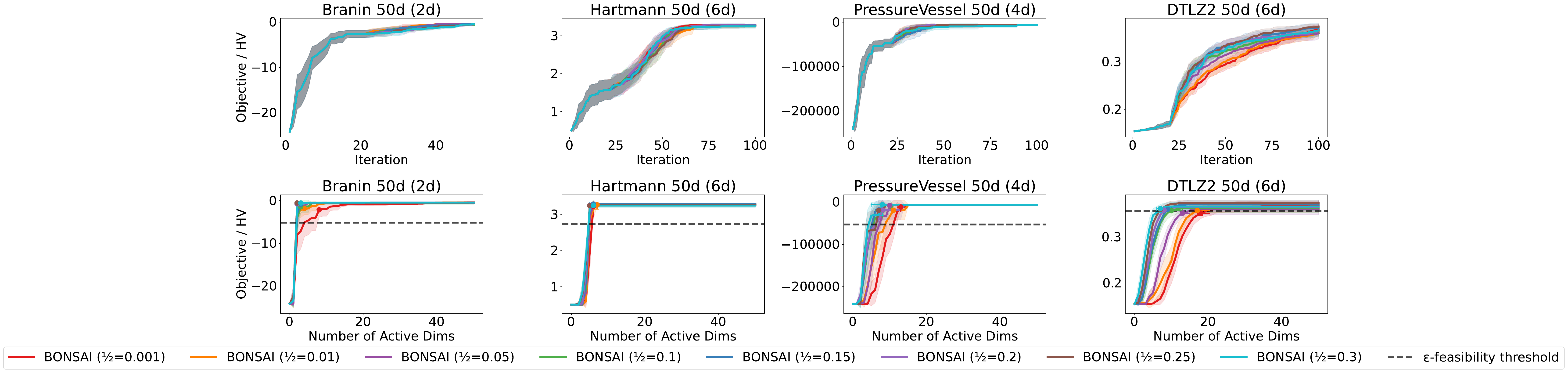}
    \caption{Top row: Objective or HV. Bottom row: Best Objective (or HV) value for each level of active dimensions. For MOO, the HV at sparsity level $k$ is the hypervolume of the feasible Pareto frontier computed over all evaluated points with at most $k$ active dimensions.}
    \label{fig:rho_sens_synthetic}
  \end{figure*}
  \begin{figure*}[h]
    \centering
    \includegraphics[width=\linewidth]{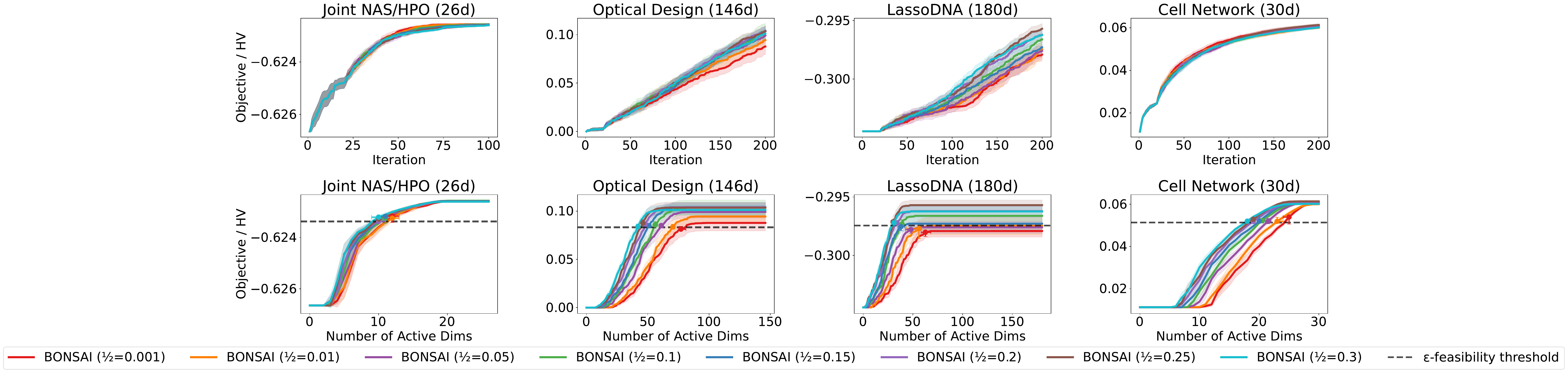}
    \caption{Top row: Objective or HV. Bottom row: Best Objective (or HV) value for each level of active dimensions. For MOO, the HV at sparsity level $k$ is the hypervolume of the feasible Pareto frontier computed over all evaluated points with at most $k$ active dimensions.}
    \label{fig:rho_sens_real}
  \end{figure*}
  \FloatBarrier
  \subsection{BONSAI with exact pruning}
  \label{appdx:bonsai_exact}
  \subsubsection{Low-dimensional benchmarks with BONSAI exact pruning}
  In this section, we evaluate Sobol, Vanilla BO, IR, ER, SEBO, and BONSAI EI with both sequential greedy and exact pruning on lower-dimensional benchmarks where exact pruning is computationally tractable. The full per-problem results in Figures~\ref{fig:bonsai_exact} and~\ref{fig:bonsai_exact2} show that BONSAI is competitive with the evaluated default-aware baselines and that its greedy and exact variants yield comparable optimization performance and sparsity. Wall times relative to Vanilla BO are reported in Table~\ref{table:gen_times_bonsai_map_saas_exact_combined}. Sequential greedy pruning becomes significantly faster than exact pruning as dimensionality increases, since exact pruning scales exponentially. Across all 7 problems and replications, the exact and greedy variants yield the same first candidate after Sobol initialization 85.7\% of the time.
  \begin{figure*}[h]
    \centering
    \includegraphics[width=\linewidth]{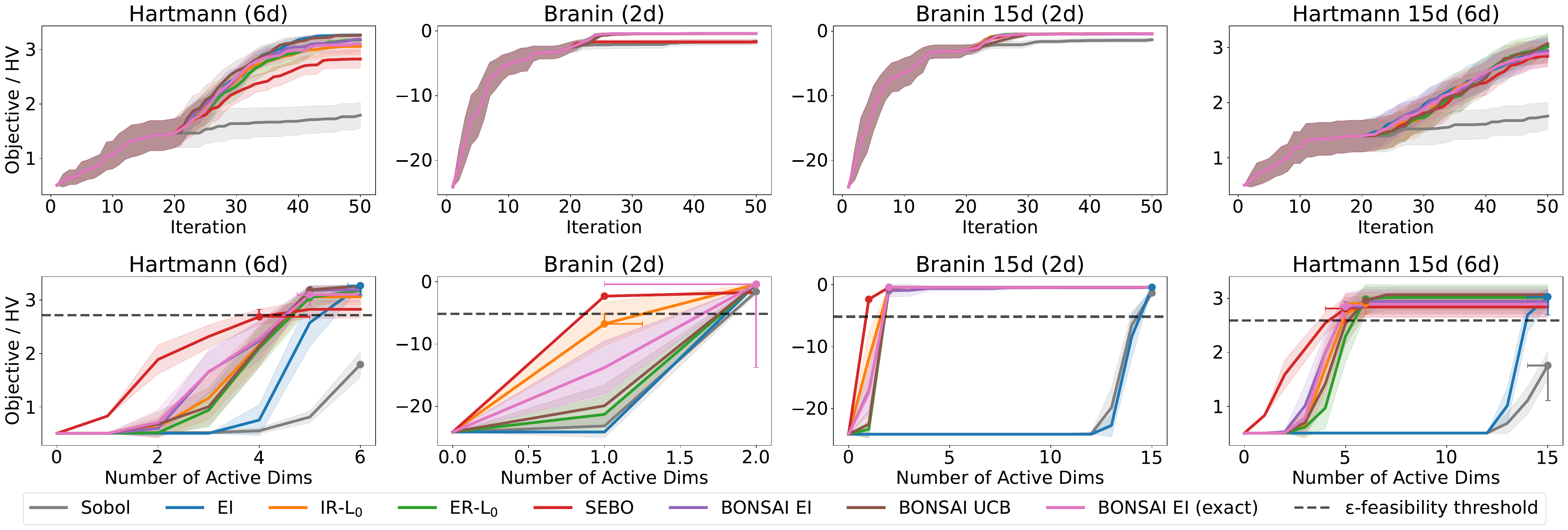}
    \caption{Optimization performance of BONSAI with exact and sequential greedy pruning with sequential optimization ($q=1$). Top row: Objective or HV. Bottom row: Best Objective (or HV) value for each level of active dimensions. For MOO, the HV at sparsity level $k$ is the hypervolume of the feasible Pareto frontier computed over all evaluated points with at most $k$ active dimensions.}
    \label{fig:bonsai_exact}
  \end{figure*}
  \begin{figure*}[h]
    \centering
    \includegraphics[width=\linewidth]{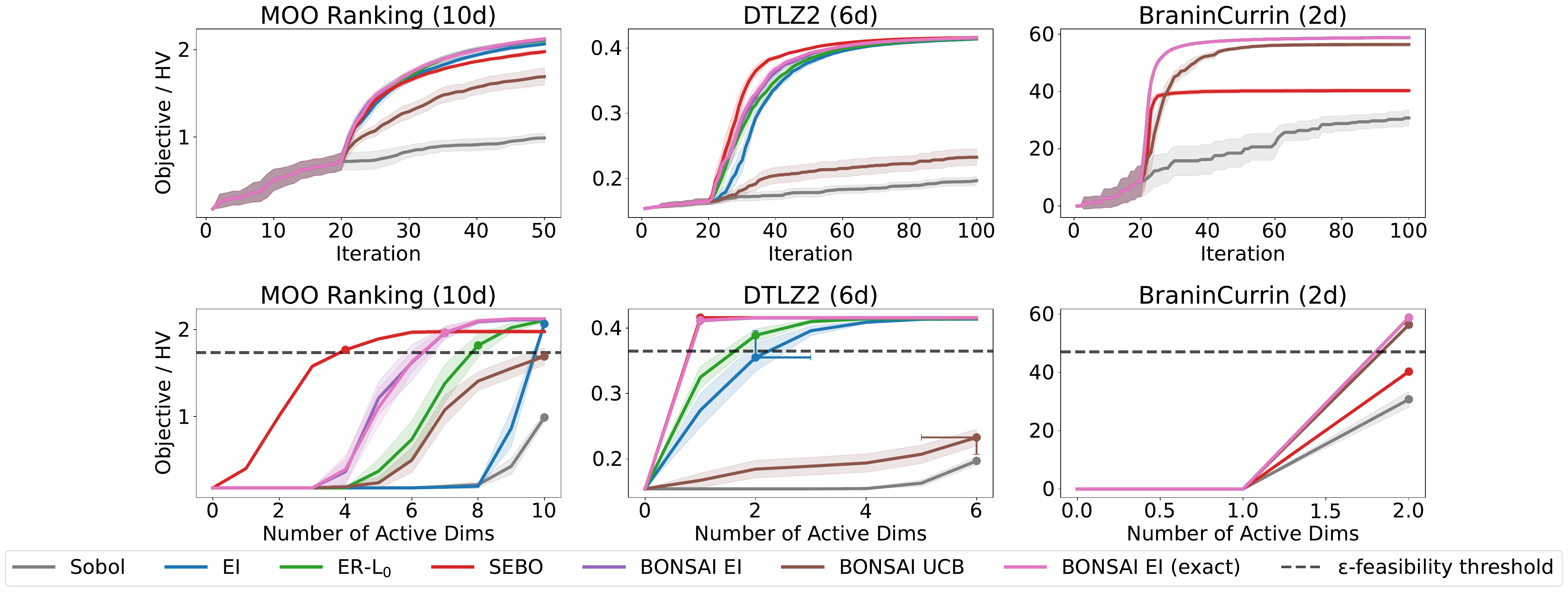}
    \caption{Optimization performance of BONSAI with exact and sequential greedy pruning with sequential optimization ($q=1$). Top row: Objective or HV. Bottom row: Best Objective (or HV) value for each level of active dimensions. For MOO, the HV at sparsity level $k$ is the hypervolume of the feasible Pareto frontier computed over all evaluated points with at most $k$ active dimensions.}
    \label{fig:bonsai_exact2}
  \end{figure*}
   \begin{table*}[t]
  \centering
  \begin{small}
  \begin{sc}
  \resizebox{\textwidth}{!}{%
  \begin{tabular}{lccccc}
  \toprule
  & IR-$L_0$ & ER-$L_0$ & SEBO & BONSAI EI & BONSAI EI (exact) \\ \midrule
  Hartmann (6d) & 3.6x ($\pm$ 0.3x) & 2.9x ($\pm$ 0.3x) & 7.2x ($\pm$ 0.9x) & 1.0x ($\pm$ 0.1x) & $\bm{1.0\text{x}}$ ($\pm$ 0.1x) \\
  Branin (2d) & 5.4x ($\pm$ 0.5x) & 5.8x ($\pm$ 0.6x) & 5.8x ($\pm$ 0.7x) & $\bm{1.0\text{x}}$ ($\pm$ 0.1x) & 1.0x ($\pm$ 0.0x) \\
  MOO Ranking (10d) & -- & 3.8x ($\pm$ 0.2x) & 5.7x ($\pm$ 0.5x) & $\bm{1.0\text{x}}$ ($\pm$ 0.0x) & 1.2x ($\pm$ 0.0x) \\
  DTLZ2 (6d) & -- & 5.6x ($\pm$ 0.2x) & 5.3x ($\pm$ 0.3x) & $\bm{1.1\text{x}}$ ($\pm$ 0.0x) & 1.2x ($\pm$ 0.0x) \\
  BraninCurrin (2d) & -- & 5.5x ($\pm$ 0.3x) & 3.0x ($\pm$ 0.2x) & $\bm{1.0\text{x}}$ ($\pm$ 0.0x) & 1.0x ($\pm$ 0.0x) \\
  Branin 15d (2d) & 3.5x ($\pm$ 0.3x) & 3.3x ($\pm$ 0.3x) & 3.9x ($\pm$ 0.5x) & $\bm{1.0\text{x}}$ ($\pm$ 0.1x) & 1.8x ($\pm$ 0.1x) \\
  Hartmann 15d (6d) & 3.7x ($\pm$ 0.1x) & 4.2x ($\pm$ 0.1x) & 6.7x ($\pm$ 0.9x) & $\bm{1.1\text{x}}$ ($\pm$ 0.0x) & 2.0x ($\pm$ 0.1x) \\
  \bottomrule
  \end{tabular}%
  }
  \end{sc}
  \end{small}
  \caption{Mean paired-seed generation-time ratio to Vanilla BO ($\pm$ 2 standard errors). The fastest default-aware method is shown in bold.}\label{table:gen_times_bonsai_map_saas_exact_combined}
  \end{table*}

  \FloatBarrier
  \subsubsection{Low-dimensional benchmarks in the batch setting ($q=5$)}
  We extend the all-methods comparison to the batch ($q=5$) setting on the same low-dimensional benchmarks where exact pruning remains tractable (BraninCurrin, Branin 15d, Hartmann 15d, etc.); exact pruning remains infeasible on the higher-dimensional benchmarks of Section~\ref{sec:experiments} since each batch element would require enumerating $O(2^{|A(\bm x_t^*)|})$ subsets. Figures~\ref{fig:bonsai_exact_q5} and~\ref{fig:bonsai_exact2_q5} show that BONSAI remains competitive with the default-aware baselines in the batch setting, and that the two BONSAI pruning variants produce comparable optimization performance and sparsity, while sequential greedy pruning remains substantially faster.
  \begin{figure*}[h]
    \centering
    \includegraphics[width=\linewidth]{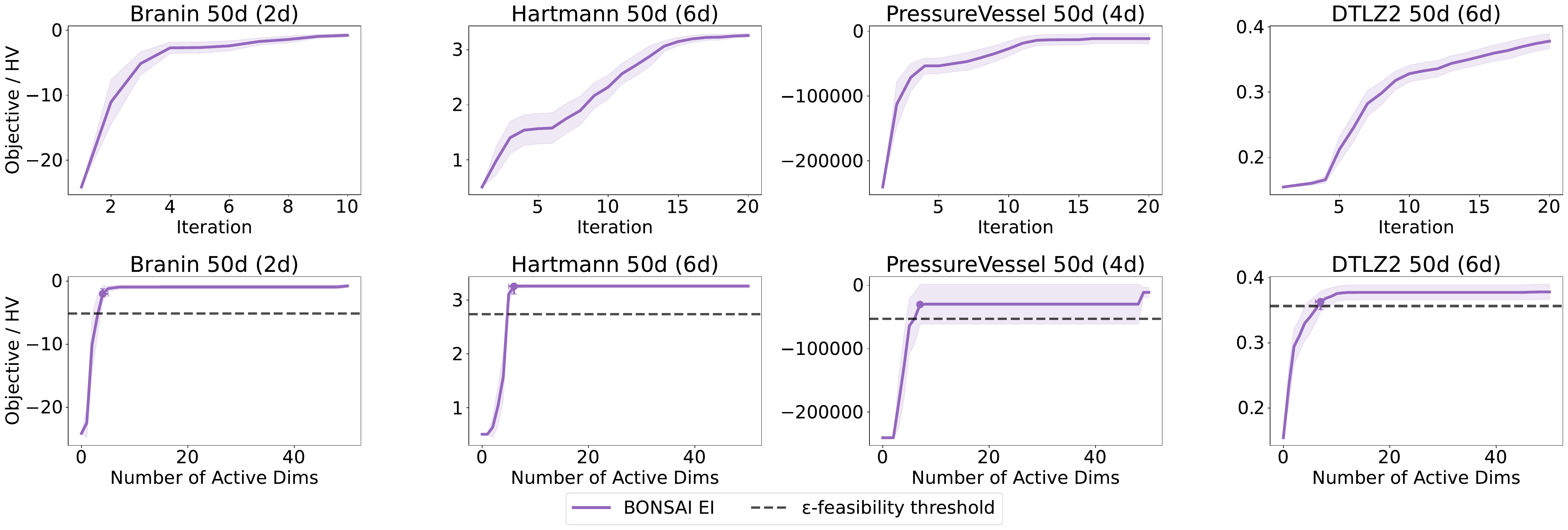}
    \caption{Optimization performance of BONSAI with exact and sequential greedy pruning with batch optimization ($q=5$). Top row: Objective or HV. Bottom row: Best Objective (or HV) value for each level of active dimensions. For MOO, the HV at sparsity level $k$ is the hypervolume of the feasible Pareto frontier computed over all evaluated points with at most $k$ active dimensions.}
    \label{fig:bonsai_exact_q5}
  \end{figure*}
  \begin{figure*}[h]
    \centering
    \includegraphics[width=\linewidth]{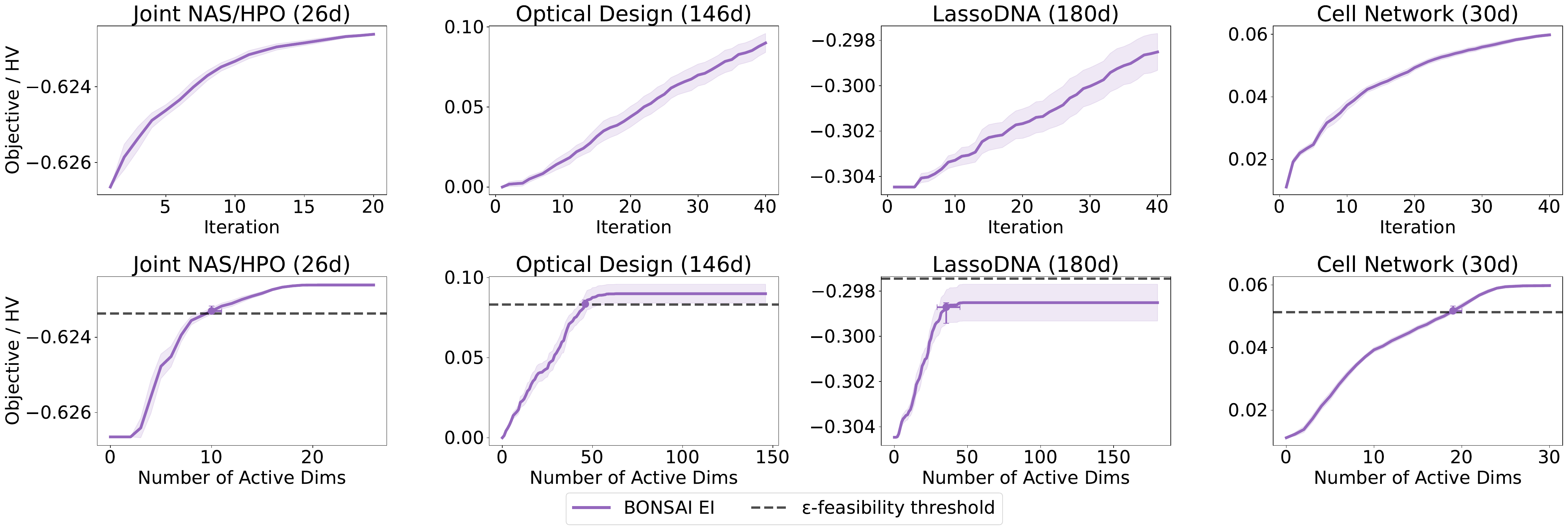}
    \caption{Optimization performance of BONSAI with exact and sequential greedy pruning with batch optimization ($q=5$). Top row: Objective or HV. Bottom row: Best Objective (or HV) value for each level of active dimensions. For MOO, the HV at sparsity level $k$ is the hypervolume of the feasible Pareto frontier computed over all evaluated points with at most $k$ active dimensions.}
    \label{fig:bonsai_exact2_q5}
  \end{figure*}
   \begin{table*}[t]
  \centering
  \begin{small}
  \begin{sc}
  \resizebox{\textwidth}{!}{%
  \begin{tabular}{lccccc}
  \toprule
  & IR-$L_0$ & ER-$L_0$ & SEBO & BONSAI EI & BONSAI EI (exact) \\ \midrule
  Hartmann (6d) & -- & -- & -- & $\bm{32.2}$ ($\pm$ 1.2) & 33.1 ($\pm$ 1.2) \\
  Branin (2d) & -- & -- & -- & 31.6 ($\pm$ 5.3) & $\bm{31.3}$ ($\pm$ 5.2) \\
  MOO Ranking (10d) & -- & -- & -- & $\bm{94.3}$ ($\pm$ 2.6) & 106.7 ($\pm$ 3.7) \\
  DTLZ2 (6d) & -- & -- & -- & $\bm{328.4}$ ($\pm$ 9.9) & 352.6 ($\pm$ 20.2) \\
  BraninCurrin (2d) & -- & -- & -- & 452.0 ($\pm$ 32.9) & $\bm{448.5}$ ($\pm$ 31.2) \\
  Branin 15d (2d) & -- & -- & -- & $\bm{52.5}$ ($\pm$ 3.2) & 95.7 ($\pm$ 3.9) \\
  Hartmann 15d (6d) & -- & -- & -- & $\bm{44.4}$ ($\pm$ 1.6) & 82.4 ($\pm$ 2.2) \\
  \bottomrule
  \end{tabular}%
  }
  \end{sc}
  \end{small}
  \caption{Average candidate generation time per iteration in seconds ($\pm$ 2 standard errors). The fastest default-aware method is shown in bold.}\label{table:gen_times_bonsai_map_saas_exact_combined_q5}
  \end{table*}

  \FloatBarrier
  \FloatBarrier
  \subsection{Dimension Scaling Priors}
  All methods achieve less sparsity with dimension scaling priors \citep{hvarfner2024dsp} than with MAP-SAAS. On problems with known numbers of active parameters, they do not find near-optimal solutions with the true number of active parameters. IR works slightly better than BONSAI on single-objective problems, but BONSAI performs best on the multi-objective problem (DTLZ2). Importantly, the optimization performance of BONSAI relative to Standard BO is unaffected by changing the model class; rather, the difference is just the level of sparsity achieved by default-aware methods.
  \begin{figure*}[h]
    \centering
    \includegraphics[width=\linewidth]{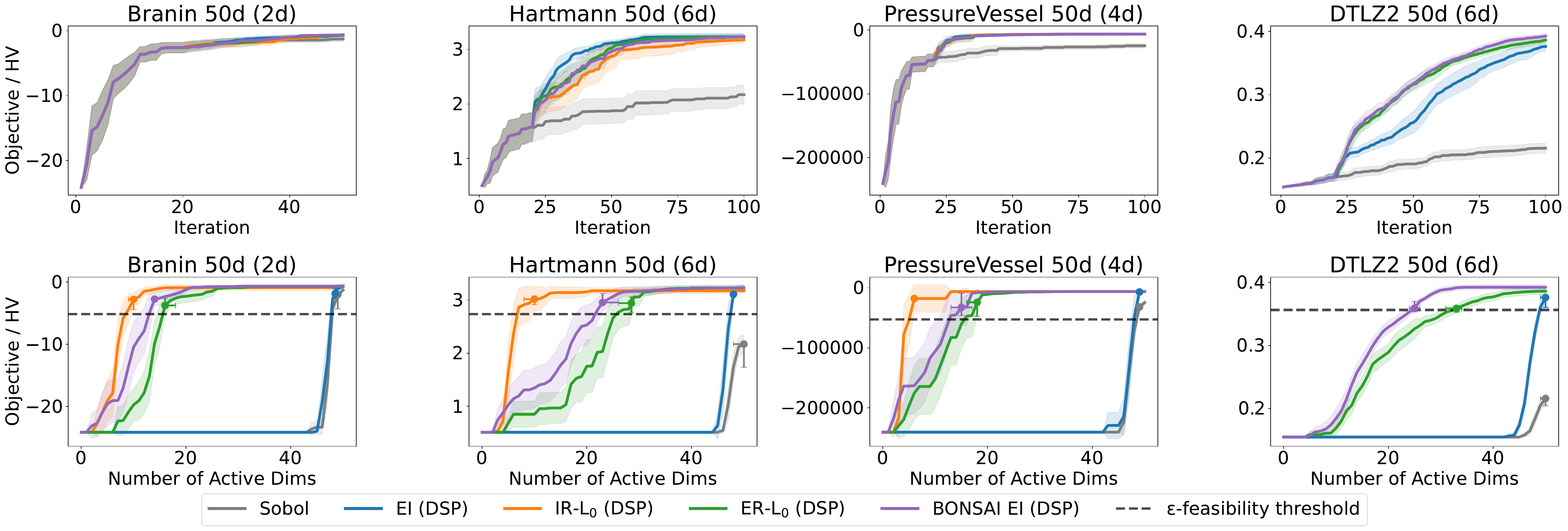}
    \caption{Sequential optimization performance of methods with a dimension-scaling prior (DSP). Top row: Objective or HV. Bottom row: Best Objective (or HV) value for each level of active dimensions. For MOO, the HV at sparsity level $k$ is the hypervolume of the feasible Pareto frontier computed over all evaluated points with at most $k$ active dimensions.}
    \label{fig:bonsai_dsp}
  \end{figure*}

  \subsection{IR/ER Sensitivity Analysis}
  \label{appdx:ir_er_sens}
  In our experiments, we set the penalty coefficient in IR and ER to be 0.01, which was a good default value from \citep{sebo}. In addition, we test different values of the penalty coefficient on synthetic problems to analyze the sensitivity of IR and ER. This analysis validates that 0.01 is a reasonable choice.
  \begin{figure*}[h]
    \centering
    \includegraphics[width=\linewidth]{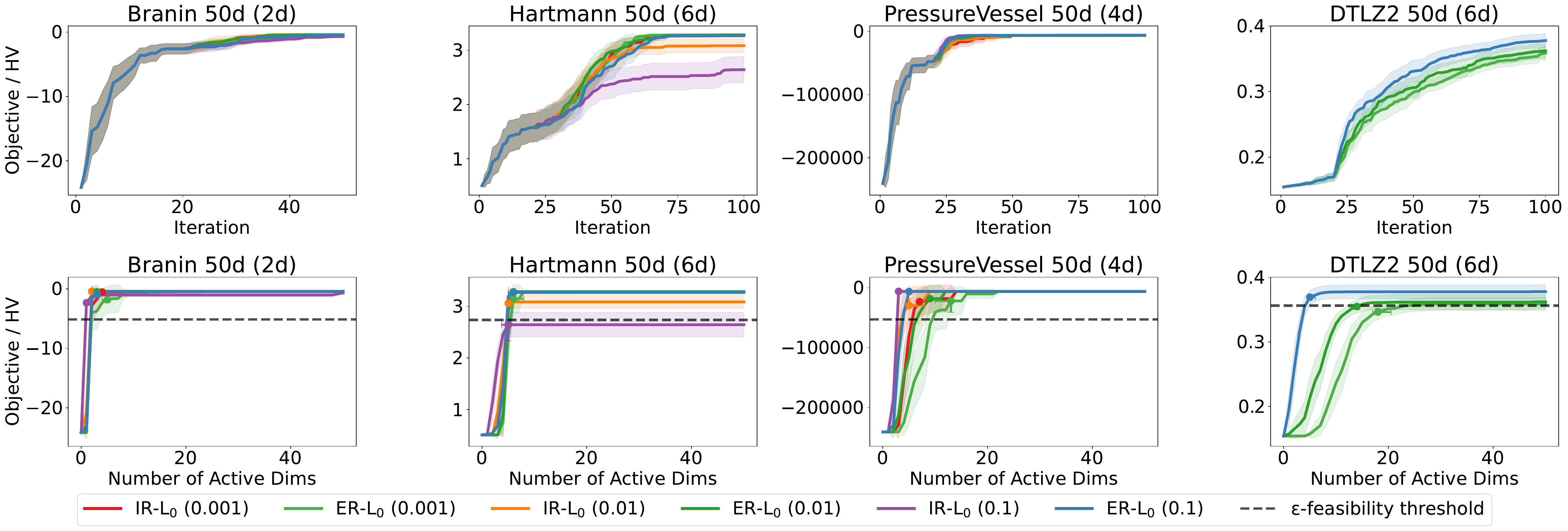}
    \caption{Top row: Objective or HV. Bottom row: Best Objective (or HV) value for each level of active dimensions. For MOO, the HV at sparsity level $k$ is the hypervolume of the feasible Pareto frontier computed over all evaluated points with at most $k$ active dimensions.}
    \label{fig:ir_er_sens_synthetic}
  \end{figure*}
  \FloatBarrier
  \subsection{Improved Lengthscale Estimation}
  \label{appdx:lengthscale_est}
  In this section we show that BONSAI improves the lengthscale estimation on the 6d DTLZ2 problem embedded in a 50d space.
  As illustrated in Figure~\ref{fig:synthetic}, BONSAI outperforms Standard BO on this problem.
  Following \citep{hvarfner2025informed}, we fit a GP with a dimension-scaling prior to the 100 trials collected by each replication of Sobol, Standard BO, and BONSAI.
  The distribution over replications of the log-lengthscales is illustrated in Figure~\ref{fig:lengthscale_est}.

  \begin{figure*}[h]
    \centering
    \includegraphics[width=\linewidth]{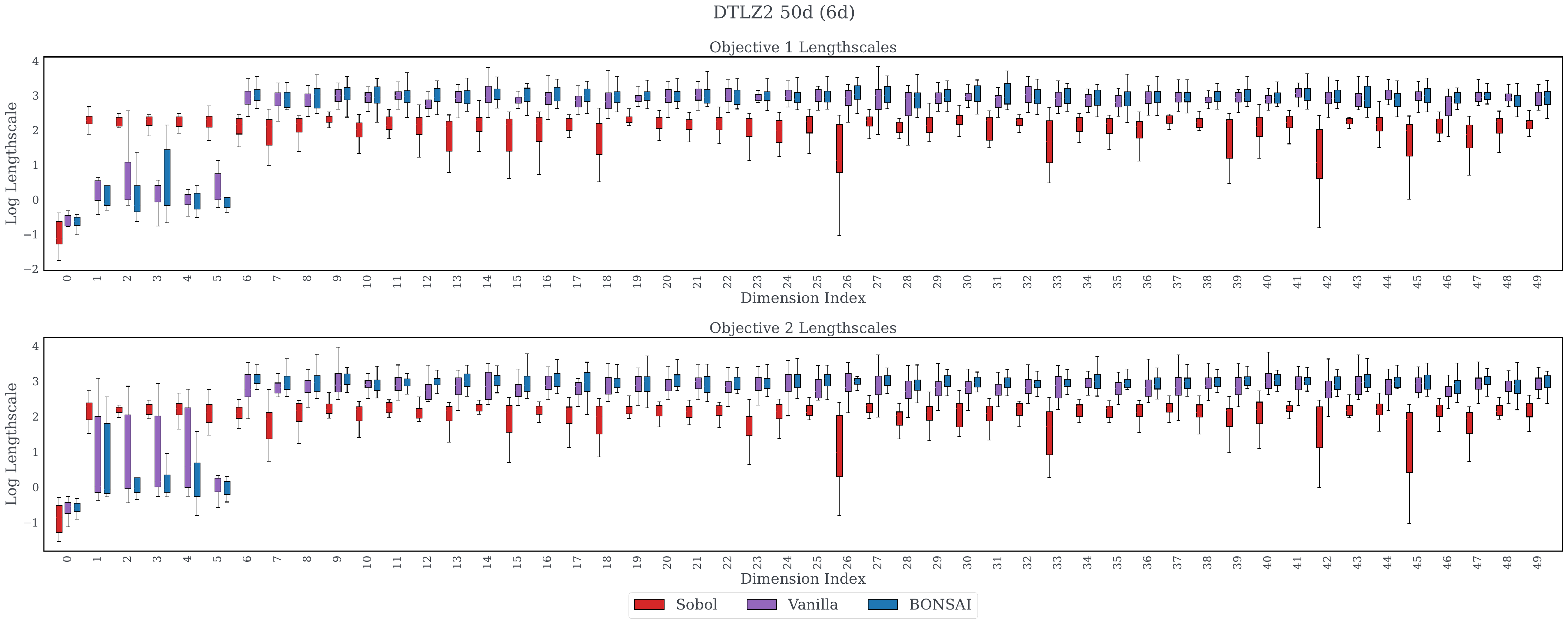}
    \caption{Estimated lengthscales using data collected by different methods on the DTLZ2 50d (6d) problem. Using the data collected by BONSAI consistently results in more accurate lengthscale estimation.}
    \label{fig:lengthscale_est}
  \end{figure*}

  We observe that we are consistently able to identify the 6 non-redundant parameters when using the data collected by BONSAI.
  In comparison, using the data collected by Sobol and Standard BO results in sometimes failing to infer small lengthscales for the 6 non-redundant parameters, demonstrating that using pruning makes it easier for the underlying GP model to infer better lengthscales.
  We believe this is one of the reasons why BONSAI outperforms Standard BO on this problem.
  \FloatBarrier
  \section{MAP Estimation of SAASBO (MAP-SAAS)}
  \label{appdx:map_saas}
  To improve scalability relative to fully Bayesian SAASBO, we use an ensemble MAP approximation in which a small number of global sparsity scales are sampled and the remaining GP hyperparameters are fit via MAP. The fitting procedure is described in Algorithm~\ref{alg:ensemble-map-saas-botorch}. This ensemble construction approximates posterior uncertainty over sparsity patterns while retaining the computational efficiency of MAP-based GP fitting. For each model in the ensemble, we use a $\text{Tophat}(10^{-2}, 10^4)$ prior over the signal variance $s_m$ and a $\text{Gamma}(0.9, 10.0)$ prior over the noise variance $\sigma_m^2$.

  \begin{algorithm}[t]
  \caption{Ensemble MAP-SAAS GP}
  \label{alg:ensemble-map-saas-botorch}
  \begin{algorithmic}[1]
  \REQUIRE Training data $X \in \mathbb{R}^{n \times d}$, $Y \in \mathbb{R}^{n \times 1}$
  \REQUIRE Ensemble size $M \leftarrow 4$

  \STATE Sample $\tau_1,\dots,\tau_M \overset{\text{iid}}{\sim} \mathrm{HalfCauchy}(0.1)$

  \STATE Form batched data $X^{(b)}, Y^{(b)} \in \mathbb{R}^{M \times n \times d}$

  \FOR{$m = 1$ {\bf to} $M$}
    \STATE Define a Mat\'ern-$5/2$ GP with ARD lengthscales $\ell_{m,1:d}$
    \STATE Use priors: $\ell_{m,j}^{-2} \sim \mathrm{HalfCauchy}(\tau_m)$ (sparsity-inducing SAAS prior on the inverse squared lengthscales), $\sigma_m^2\sim \text{Gamma}(0.9, 10.0)$, $s_m\sim \text{Tophat}(10^{-2}, 10^4)$
    \STATE Fit $\{\ell_{m,j}, \sigma_m^2, s_m\}$ by MAP
  \ENDFOR

  \STATE Return a uniform Gaussian mixture posterior over ensemble members
  \end{algorithmic}
  \end{algorithm}

  Additionally, we compare BONSAI with MAP-SAAS to BONSAI with a GP with a dimension-scaling prior (DSP) \citep{hvarfner2024dsp} and SAASBO \citep{saasbo} and find that MAP-SAAS induces more sparsity.

  \begin{figure*}[h]
    \centering
    \includegraphics[width=\linewidth]{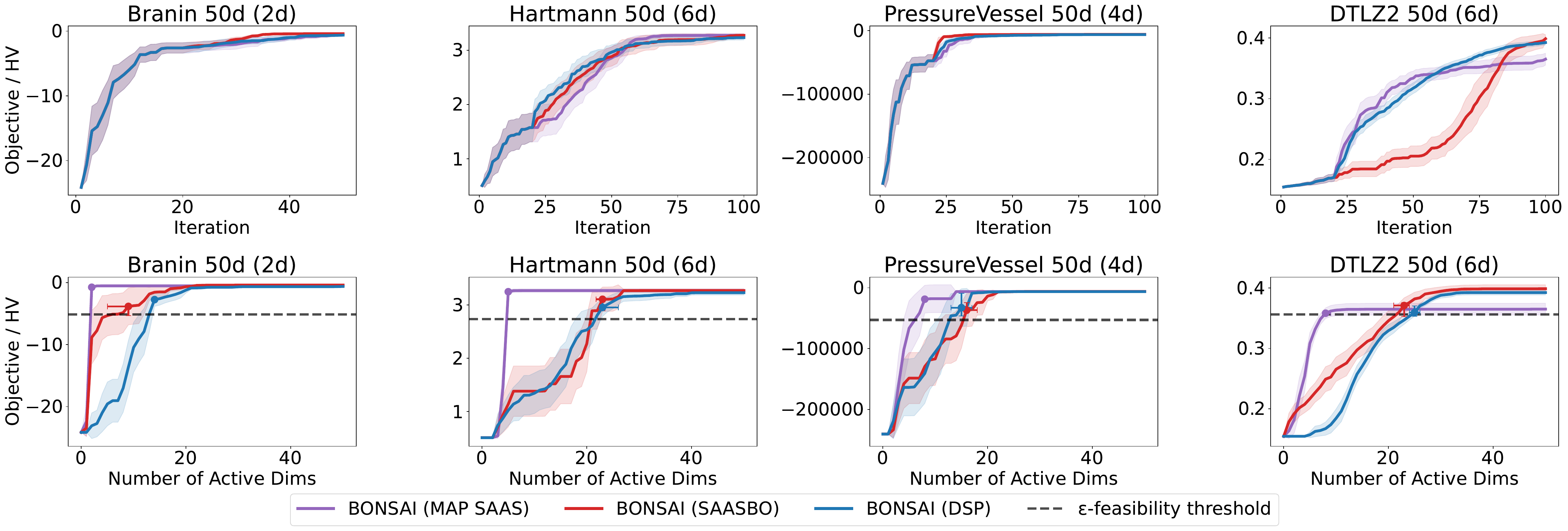}
    \caption{Sequential optimization performance of BONSAI on synthetic problems with MAP-SAAS and the dimension-scaling prior. Top row: Objective or HV. Bottom row: Best Objective (or HV) value for each level of active dimensions. For MOO, the HV at sparsity level $k$ is the hypervolume of the feasible Pareto frontier computed over all evaluated points with at most $k$ active dimensions.}
    \label{fig:map_saas_vs_dscaling}
  \end{figure*}

  \FloatBarrier
  \section{Extensions of IR and ER}
  \label{appdx:ir_er_exts}
  IR can be extended to constrained qLogNEI in a straightforward way, since IR changes the objective by adding a penalty term based on the $\ell_0$-norm. Similarly, ER can also be extended to constrained qLogNEI since the penalty in ER is simply added to the acquisition function externally.

  Extending IR to multi-objective problems is non-trivial. For IR, one could penalize each objective by the sparsity penalty, but that would have a multiplicative effect on the hypervolume. Alternatively, one could subtract the sparsity penalty from the hypervolume of a set of points, but that would require extending the penalty to apply to sets of points. The issues of scale between hypervolume and the sparsity penalty may make it difficult to choose a reasonable value of the sparsity penalty coefficient.

For ER in the multi-objective setting, we apply the $\ell_0$ penalty externally to qLogNEHVI: at each acquisition optimization, we compute $\alpha^{\mathrm{ER}}_t(\bm x) = \mathrm{qLogNEHVI}_t(\bm x) - \lambda \cdot \|\bm x - \bm x^{\mathrm{def}}\|_0$, where the penalty coefficient $\lambda$ is set to the same default value $0.01$ used in single-objective ER. We apply the penalty after the log transform (i.e., directly to qLogNEHVI's output rather than to the raw NEHVI), so that the penalty acts on the same scale as the log-acquisition. As in single-objective ER, we use 30 homotopy continuation steps with $a_\text{start}=0.2$ and $a_\text{end}=10^{-3}$ following the recommendation of \citet{sebo}.

  \section{Additional Related Work}
  \label{appdx:additional_related_work}

  \paragraph{Conservative and safe bandits with baselines.}
  A related line of work studies \emph{conservative} or \emph{safe} policies that constrain performance relative to a baseline arm or configuration, e.g., by requiring the cumulative reward not to fall too far below a status-quo level~\citep[e.g.,][]{kazerouni2017conservative,sui2015safe}. BONSAI is not a safety mechanism of this particular type: our constraint is imposed on the acquisition rather than on the unknown objective, and the primary goal is to simplify recommendations in input space rather than to provide outcome-level safety guarantees. We view these approaches as complementary, and BONSAI could be combined with conservative BO when both simplicity and outcome-level constraints are required.

  \paragraph{Variable selection and shrinkage for high-dimensional BO.}
  Other work exploits sparsity to improve BO in high dimensions. Methods such as SAASBO~\citep{saasbo} and VSBO~\citep{vsbo} aim to identify a small set of influential dimensions via shrinkage priors on kernel lengthscales or forward-stage selection using importance scores. \citet{hvarfner2024dsp} propose a dimensionality-scaled hyperprior whose mean grows with $\sqrt{d}$, progressively shrinking lengthscales toward larger values as dimensionality increases. These techniques primarily target the structure of the objective $f$ to improve sample efficiency. However, even when only a few dimensions are relevant, the recommended configuration can still deviate substantially from a practitioner-chosen default across many components. BONSAI instead measures complexity directly via the $\ell_0$ distance to a user-specified default and prunes low-impact deviations regardless of whether they arise in globally relevant or irrelevant dimensions.

  \paragraph{Explainable and interpretable BO.}
  Recent work has begun to study the interpretability of BO recommendations. For example, \citet{chakraborty2025explainablebayesianoptimization} explain BO solutions by identifying tunable subspaces and quantifying parameter-wise contributions, and \citet{pmlr-v238-adachi24a} explore techniques to make BO behavior more transparent to practitioners. BONSAI is complementary: it provides operational simplicity by reverting deviations whose model- and acquisition-conditioned loss is within tolerance. It does not provide causal attribution or a mechanistic explanation of why the objective changes. The pruned suggestion can instead serve as a smaller input to post-hoc explanation methods.

  \paragraph{Cost-aware and switching-cost Bayesian optimization.}
  Several works incorporate explicit costs into BO. Cost-aware BO methods modify the acquisition to trade off reward and cost, or to preferentially query cheap configurations; see, e.g., work on cost-varying or subset-aware BO~\citep{bocvs,mcts_vs}. \citet{switching_costs} consider modular black-box systems in which changing components incurs switching costs and optimize a cost-augmented objective.
  BONSAI differs from these approaches in two ways. First, it does not require specifying or learning a cost function over individual input parameters. Instead, it enforces a hard constraint on the \emph{acquisition gap} relative to the maximizer and, within that constraint, explicitly prefers the candidate with the fewest deviations from a fixed default configuration. Second, BONSAI is implemented as an acquisition-conditioned candidate transformation: it leaves the surrogate and acquisition function unmodified but changes the point that is evaluated, and can thus be applied across a broad range of settings.
  \paragraph{Lengthscale-thresholding as a pruning baseline.}
  A natural alternative to BONSAI would be to eliminate changes in dimensions whose estimated ARD lengthscale exceeds a threshold. However, this approach has important limitations: (1) it is not straightforward to apply in multi-output settings (multi-objective and constrained problems) where multiple GPs with different lengthscale estimates must be reconciled; and (2) choosing a lengthscale threshold is itself a sensitive hyperparameter---for constrained problems, a lengthscale that appears large for the objective GP may correspond to a dimension critical for constraint satisfaction. In contrast, BONSAI's relative acquisition tolerance $\rho$ has a natural interpretation across all problem types (single-objective, constrained, and multi-objective), and we demonstrate its robustness to the choice of $\rho$ in Appendix~\ref{appdx:bonsai_rho_sens}.
  \section{Future Work}
  \label{appdx:future_work}
  In this section, we discuss directions for future work.

  BONSAI measures simplicity using coordinate-wise $\ell_0$ distance to the default, which is appropriate when each coordinate corresponds to a semantically meaningful and independently actionable knob. When parameters are correlated or meaningful only as groups (common in systems tuning), a grouped or weighted sparsity notion can be more realistic; we view BONSAI as a first step toward such structured pruning. Empirically, we also compare against exact pruning, which can reset multiple knobs simultaneously in one-shot and is capable of group-level pruning, and we find similar performance to sequential greedy pruning (Appendix~\ref{appdx:bonsai_exact}).

  There are several directions for future work. First, the present evaluation uses one semantically selected default per problem and does not relocate it systematically. Controlled experiments that vary the quality, location, and distance of the supplied default from high-performing regions would clarify when default-relative sparsity is most useful and how robust BONSAI and prior-centered baselines are to misspecified defaults. Second, tighter sparsity guarantees should move beyond the simplified per-component perspective in \cref{appdx:sparsity} and better characterize when greedy pruning matches the combinatorial optimum. Third, extending the theoretical analysis to EI and additional acquisition functions would require new tools, since the current proof techniques rely on confidence bounds. Finally, more sophisticated models of user preferences---including weighted or grouped parameters and structured defaults---and combinations with post-hoc explanation methods may further improve transparency.

  \section{Impact Statement}

  This work aims to improve the usability and interpretability of Bayesian optimization in applications where there is a meaningful default configuration and where changing many parameters at once can be risky or costly from an operational perspective. By encouraging recommendations that change as few input parameters as possible relative to the default, BONSAI may help practitioners deploy BO in operationally constrained settings, such as large-scale systems configuration or model serving infrastructure.

  At the same time, BONSAI is not a substitute for careful specification of objectives and constraints: if important aspects of system behavior are omitted from the optimization objective, BO---with or without BONSAI---may still propose configurations that are undesirable in practice. We therefore view BONSAI as a complementary tool that can make it easier for practitioners to inspect and reason about BO suggestions, but it should be used in conjunction with domain expertise, monitoring, and appropriate validation checks.
  % \newpage
  % \input{checklist.tex}

  \end{document}